\documentclass{article}
\usepackage{mathptmx}

 \usepackage{framed}
\usepackage{times}
\usepackage{amsfonts}
\usepackage{amsmath}
\usepackage{amssymb}
\usepackage{mathrsfs}
\usepackage{rotating}
\usepackage{latexsym}
\usepackage{xspace}
\usepackage{color}
\usepackage{graphicx}
\usepackage{rotating}
\usepackage{times}

\usepackage{amsmath}
\usepackage{amsfonts}
\usepackage{amssymb}
\usepackage{mathrsfs}

\usepackage{xspace}

\usepackage{enumerate}
\usepackage{hyperref}
\usepackage{color}

\newcommand{\tagLabel}[2]{\tag{\textbf{#1}}\label{#2}}

\newenvironment{proof}{\medskip\noindent \textsc{Proof.}}
{\hspace*{\fill}\nolinebreak[2]\hspace*{\fill}$\blacksquare$\medskip}

\newenvironment{example}
{
 \noindent \textbf{Example (cont.)}\begin{itshape}%
}%
{\end{itshape}}

\newtheorem{definition}{Definition}
\newtheorem{proposition}{Proposition}
\newtheorem{theorem}{Theorem}

\newtheorem{lemma}{Lemma}

\begin{document}

  \title{A Qualitative Theory of Cognitive Attitudes \\
        and their Change}

  \author{EMILIANO LORINI\\
         IRIT-CNRS, Toulouse University, France}

 \date{}

\maketitle

\begin{abstract}

We present a general logical framework for reasoning about 
agents' cognitive attitudes of both
epistemic type and motivational type.
We show that it allows 
us to express a variety of relevant concepts 
for qualitative decision theory including the concepts of 
knowledge, belief, strong belief, conditional belief, desire, conditional desire, strong desire
and preference. We also present two extensions of the logic,
one by the notion of choice
and the other by dynamic operators for belief change and desire change,
and we apply the former to the analysis of single-stage games under incomplete information.
We provide sound
and complete
axiomatizations for the basic logic
and for its two extensions. 

The paper is 
``under consideration in Theory and Practice of Logic Programming (TPLP)''.

\end{abstract}

\section{Introduction}\label{intro}

 Since the seminal work of
 Hintikka
 on epistemic logic  \cite{Hintikka},
 of Von Wright on the logic
 of preference  \cite{WrightPreference,WrightPreference2}
 and
 of Cohen \& Levesque
 on the logic
 of intention  \cite{Coh90}, many formal logics for reasoning about cognitive
attitudes of agents such as knowledge and belief
\cite{Fagin1995}, preference \cite{Fenrong,BenthemGirardRoy}, desire \cite{DuboisLorini},
intention \cite{Shoham2009,IcardPacuit}
and their combination
\cite{Mey99,Woo00}
 have been proposed.
 Generally speaking, these logics are nothing but formal models of rational agency relying on the idea that
 an agent endowed
with cognitive attitudes makes decisions on the basis of what she believes and of what she desires or prefers.

The idea of describing rational agents in terms of their epistemic and motivational attitudes
is something that these logics share with classical decision theory and game theory.
Classical decision theory and game theory provide a quantitative account
of individual and strategic decision-making by assuming that agents' beliefs and desires can be respectively modeled
by subjective probabilities and utilities.
Qualitative approaches to individual and strategic decision-making have been proposed in AI
\cite{Boutilier94,DoyleThomason}
to characterize criteria that a rational agent should adopt for making decisions when she cannot build a probability distribution
over the set of possible events and her preference over the set of possible outcomes cannot be expressed by a utility function but
only by a qualitative ordering over the outcomes. For example, going beyond expected utility maximization, qualitative criteria such as
the maxmin principle (choose the action that will minimize potential loss) and the maxmax principle
(choose the action that will maximize potential gain) have been studied and axiomatically characterized
\cite{BrafmanTennen,BrafmanTennen2}.

The aim of this paper is to present an expressive logical framework
for representing
both the static and the dynamic aspects
of a rich
 variety of agents' cognitive attitudes  in a multi-agent setting. 
In agreement with
philosophical
theories  \cite{Platts,Searle1979,HumberstoneFit,LoriniIfColog},
our logic allows us to  distinguish two general categories
of cognitive attitudes:
\emph{epistemic}
attitudes, including belief and knowledge,
and 
\emph{motivational}
ones, including desire and preference. 
Moreover, 
 in agreement with rational choice theory,
 it allows us to capture a notion of choice
 which depends on what
 an agent believes and prefers.\footnote{
Rational choice theory (RCT)
is a umbrella
term
for a family of theories
prescribing  that
 an agent should choose the course of action that,
 according to her beliefs,
 leads to the most desirable
 (or most preferred) consequences. In other words, RCT relies on the general
 assumption that agents 
 make optimal choices
 in the light of her
 beliefs, desires and preferences.
 See \cite{Paternotte}
 for more details on RCT.
 }

The example depicted in Figure 
 \ref{fig:TrafficExample}
 brings to the fore the epistemic
 and motivational  attitudes
 that are involved
 in everyday situations
 whereby artificial  
 agents
 are supposed to interact. 
There are two autonomous agents
meeting at  a crossroad:
agent $1$ and agent $2$.
The two agents could be either two mobile robots or two autonomous vehicles. 
Each agent  can decide either to stop or to continue.
 If an agent stops, then it 
will lose time. 
If both agents decide to continue,
they will collide
and, consequently, each of them
will lose time.
Therefore, for an agent
not to  lose time, it has to continue,
while the other agent decides to stop.

In this situation, each agent 
is identified 
with the set of cognitive attitudes
it endorses.
For instance, it is reasonable
to suppose that the two agents
know that 
in the situation they face
 necessarily 
some of them
will lose  time
and that 
if
one of them
 loses time
by letting the other pass,
there will be no collision.
On the motivational side,
it is reasonable to suppose
that 
each agent 
is strongly motivated by two
desires, namely, 
the desire
not to lose time
and the desire
to avoid a collision.

\begin{figure}[h]
\centering
    \includegraphics[scale=0.40]{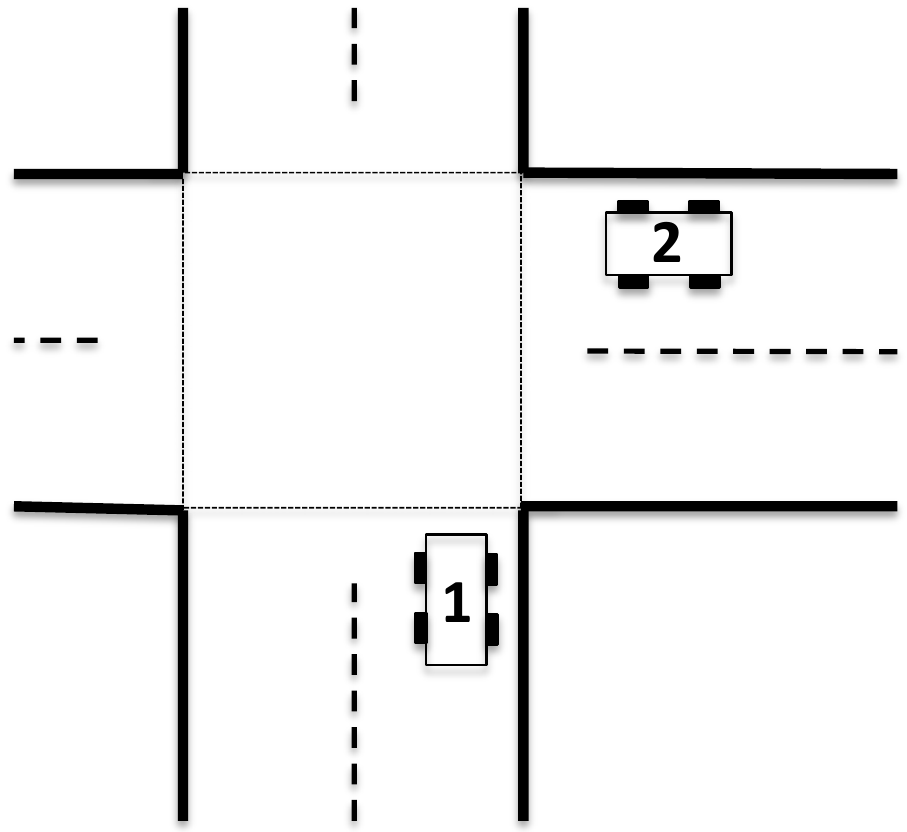}
\caption{Crossroad game}
\label{fig:TrafficExample}
\end{figure}

 On the dynamic side,
 we consider two basic
 forms of cognitive attitude change,
 namely, belief change and desire change.
 While belief change has been extensively
 studied in the area of belief
 revision \cite{Alc85,DBLP:conf/ijcai/Boutilier93,DBLP:journals/ai/DarwicheP97,Spohn,Rott27,DuboisBR} and dynamic epistemic logic 
  ($\mathrm{DEL}$) \cite{BenthemRevision,BaltagSmets2008,DBLP:journals/synthese/Ditmarsch05,DBLP:conf/prima/Aucher04},
  desire change is far less studied and understood.
We will
 study two basic forms
 of cognitive attitude revision, namely,
  \emph{radical} attitude revision
  and   \emph{conservative} attitude revision.
  While the distinction
  between radical and conservative
  belief revision has been drawn
  before (see, e.g., \cite{BenthemRevision}),
  the distinction
  between radical and conservative desire
  revision is new.
  Radical belief revision by an input $\varphi$
  makes all states at which $\varphi$ is true 
  more plausible than all states at which $\varphi$ is false,
  whereas 
  conservative belief revision by $\varphi$
  simply promotes the most plausible states
  in which $\varphi$
  is true to the highest plausibility 
rank, but apart from that, it keeps the old plausibility ordering. 
For example, 
suppose 
in the crossroad game of Figure
\ref{fig:TrafficExample},
 agent $1$
 and agent $2$
can communicate.
Agent $1$
informs
agent $2$
that ``if they both
lose time,
then there will no collision''
and agent $2$
trusts what agent $1$
says.
Then, by performing a conservative
belief
revision, agent $2$
will promote
the most plausible
situations
in which 
the formula
announced by $1$
 is true 
 to the highest plausibility rank.
 As a consequence,
agent $2$
will start to believe
what $1$
has just said.

  Symmetrically, 
  radical desire revision by  $\varphi$
  makes all states at which $\varphi$ is true 
  more desirable than all states at which $\varphi$ is false,
  whereas 
  conservative desire revision by $\varphi$
  simply demotes the least desirable states in which $\varphi$
  is false
 to the lowest desirability rank, but apart from that, it keeps
the old desirability  ordering.
For example, 
suppose 
in the crossroad game
 agent 
$1$ has just
learnt
that agent $2$
is an ambulance which 
has to transport a patient to the hospital as quickly  as possible.
Consequently,
$1$
starts to be
altruistically
 motivated
by the fact
that  $2$
does not lose time.
Thus, by performing a radical
desire
revision, agent $1$
will start to consider all situations
in which 
$2$
does not lose time
more desirable
than the situations in which
it does. This radical desire revision
operation
leads
agent $1$
to strongly desire
that 
agent $2$
does not lose time.

 The paper is organized as follows. 
 In Section \ref{LogicPart},
 we present the semantics
 and syntax of our logic,
 called 
 Dynamic Logic of Cognitive Attitudes  ($\mathrm{DLCA}$).
  At the semantic level, it exploits
  two orderings that
  capture, respectively,
 an agent's comparative plausibility
  and comparative desirability
  over states. 
 At the syntactic level, it uses
  program constructs 
 of dynamic logic (sequential composition,
 non-deterministic choice, intersection, complement, converse and test) to build
 complex cognitive attitudes
 from simple ones.  
Following \cite{PassyTinchev,Goranko93}, it also
 exploits nominals in order to axiomatize  intersection
 and complement of programs.   
  In Section \ref{FormalizationPart},
  we illustrate the expressive
  power of our logic by using it to formalize
  a variety of  cognitive attitudes 
  of agents including
  knowledge, belief, strong belief, conditional belief, desire, strong desire,
  conditional desire and preference. 
  We instantiate
  some
  of these concepts
  in the crossroad game
  depicted
in Figure \ref{fig:TrafficExample}.
   In Section \ref{AxiomatizationPart},
   we present a sound and complete axiomatization
   for our logic.
   In 
   Section \ref{sec:GTapplication}  we present the first extension of our logic
   by the notion of choice and  apply it to the analysis of single-stage games under incomplete information.
Section  \ref{sec:dynext} presents the second  extension
of our logic by dynamic operators for belief and desire change. 
   In Section \ref{ConclusionPart} we conclude. Formal proofs are given in a technical annex
   at the end of the paper.\footnote{This paper is an extended and improved
   version of \cite{DBLP:conf/jelia/Lorini19}. The JELIA'19 paper did not include the two extensions of Section
   \ref{sec:GTapplication}  and Section  \ref{sec:dynext}, or the detailed proof of the completeness theorem for the logic $\mathrm{DLCA}$.
   Also, the logical analysis of the cognitive attitudes in Section 
   \ref{FormalizationPart} has been extended: (i) we included the notion of conditional desire
   which was not considered in the JELIA'19 paper, and (ii)
   we added new logical validities
 which describe interesting properties of cognitive attitudes.
   
     }

\section{Dynamic Logic of Cognitive Attitudes }\label{LogicPart}

Let $\mathit{Atm}$
be a countable infinite set of atomic
propositions,
let
$\mathit{Nom}$
be a countable infinite set of nominals
disjoint from  $\mathit{Atm}$
and let 
$\mathit{Agt}$
be a finite set of agents.

\begin{definition}[Multi-agent cognitive model]\label{DefSemanticsModel}
A multi-agent cognitive model (MCM) is a tuple
$M = (  W,( \preceq_{i,P})_{i \in \mathit{Agt} }, (\preceq_{i,D})_{i \in \mathit{Agt} }  ,
(\equiv_i)_{i \in \mathit{Agt} }  , V  )$
where:
\begin{itemize}
\item  $W$ is a set of worlds or states;
\item  
for every $i \in \mathit{Agt} $,
$\preceq_{i,P} $ and $\preceq_{i,D}$ are  preorders on $W$
and $\equiv_i $ is an equivalence relation on $W$
such that for all $\tau \in \{P,D\}$ and   for all $w, v \in W$:
\begin{description}
\item[(C1)] $\preceq_{i,\tau} \subseteq \equiv_i $,

\item[(C2)]  if $w\equiv_i v$ then $w \preceq_{i,\tau}  v$ or $v \preceq_{i,\tau}  w$;

\end{description}

\item $V :W  \longrightarrow
2^{\mathit{Atm} \cup
\mathit{Nom} }$ is a valuation function
such that  for all $w,v \in W$:
\begin{description}

\item[(C3)]  $  V_{\mathit{Nom}}(w)  \neq \emptyset$,

\item[(C4)] if $ V_{\mathit{Nom}}(w) \cap V_{\mathit{Nom}}(v)  \neq \emptyset $
 then $w =v$;

\end{description}
where $V_{\mathit{Nom}}(w) = \mathit{Nom} \cap V(w)$.

\end{itemize}
\end{definition}

$w \preceq_{i,P} v$
means that, according to agent $i$, $v$
is  at least as plausible as $w$,
whereas
$w \preceq_{i,D} v$
means that, according to agent $i$, $v$
is at least as desirable as $w$.
Finally,
$w \equiv_i v$
means that $w$
and $v$
are indistinguishable for agent $i$.
For every $w \in W$,
$\equiv_i\!\! (w)$
is also called agent $i$'s
information set at state $w$.
According to Constraint C1,
an agent can only compare
the plausibility (resp. desirability)
of two
states in her information set. 
According to Constraint C2,
the plausibility (resp. desirability)
of
 two states
in an agent's information set
are always comparable.
Constraints C3 and C4 capture the two basic properties of nominals:
every state is associated with at least one nominal and
there are no different states
associated with the same nominal.



 Note that there is no connection between binary relations 
 $\preceq_{i,P}$
 and $\preceq_{i,D}$.
 In accord
 with
 classical
 decision and game theory
 in which an agent's subjective probability
 and utility function do not interact,
 we adopt a normative
 view of 
 epistemic and motivational
attitudes 
according to which 
 an agent's 
 epistemic plausibility and
desirability
  are assumed
 to be independent.\footnote{The normative
 view
 is usually opposed to the descriptive view.
 The
 normative view is aimed
 at describing the reasoning and decision-making
 of  ideal agents conforming to
  standards of rationality, while 
  the descriptive view
 is concerned with
 psychologically realistic cognitive agents 
 who systematically violate standards
 of rationality and exhibit different types of cognitive bias.
 }
 Therefore,
 we do not consider
 cognitive  biases
 typical
 of human reasoning
 such as 
wishful thinking.
 as the tendency to form beliefs 
 according to what is desired 
 in the absence of a clear
 evidence against it \cite{MarshWallaceWT}.
 Nonetheless,
 as we will show in Section \ref{choiceSect},
 the primitive relations
  $\preceq_{i,P}$
 and $\preceq_{i,D}$
 can be combined to obtain 
 a notion of realistic preference
 which is essential
 for elucidating the connection
 between
 an agent's beliefs and desires and 
 her  choices.

We introduce the following modal
language $\mathcal{L}_{\mathrm{DLCA} } ( \mathit{Atm},  \mathit{Nom},  \mathit{Agt})$,
or simply $\mathcal{L}_{\mathrm{DLCA} }$,
for the Dynamic Logic of Cognitive Attitudes
 $\mathrm{DLCA}$:
\begin{center}\begin{tabular}{lcl}
  $\pi$  & $::=$ & $    \equiv_i \mid \preceq_{i,P} \mid \preceq_{i,D} \mid
\preceq_{i,P}^\sim \mid \preceq_{i,D}^\sim    \mid 
   \pi { ; } \pi'
  \mid
  \pi \cup \pi' \mid \pi \cap \pi' \mid
  -\pi  \mid \varphi?
   $\\
 $\varphi$  & $::=$ & $ p \mid
 x
 \mid \neg\varphi \mid \varphi\wedge\varphi' \mid [\pi] \varphi
$
\end{tabular}\end{center}
where $p$ ranges over $ \mathit{Atm} $,
$x$ ranges over $ \mathit{Nom} $ and
$i$ ranges over $ \mathit{Agt} $.
 The other Boolean constructions 
 $\top$,
 $\bot$,
 $\vee$,
 $\rightarrow$
 and
 $\leftrightarrow$
 are defined 
  from $p$, $\neg$ and $\wedge$ in the standard way.
  The propositional language 
  built from
  the set of atomic propositions $\mathit{Atm}$
  is noted $\mathcal{L}_{\mathsf{PL}}(\mathit{Atm})$. 
Note that
the sets 
$\mathit{Atm}$,
$\mathit{Nom}$
and
$\mathit{Agt}$
define the signature of the language
$\mathcal{L}_{\mathrm{DLCA} }$.
They are
not part of the model
since every atomic
proposition $p$,
nominal $x$
and modal formula
$ [\pi] \varphi$
should be interpretable
relative to any MCM.

Elements $\pi$
are called \emph{cognitive programs}
or, more shortly, \emph{programs}.
The set of all programs is noted 
$\mathcal{P}(\mathit{Atm},\mathit{Nom},\mathit{Agt})$, or simply,
$\mathcal{P}$.

Cognitive programs correspond
to the basic constructions of Propositional Dynamic Logic ($\mathrm{PDL}$) \cite{Har00}:
atomic programs of type 
$ \equiv_i $,
$  \preceq_{i,P} $,
$   \preceq_{i,D} $,
$\preceq_{i,P}^\sim$ and
$ \preceq_{i,D}^\sim$,
sequential composition ($;$),
non-deterministic choice ($\cup$),
intersection ($\cap$),
converse ($^-$)
and test ($?$).
A given cognitive program $\pi$ corresponds to
a specific configuration of the agents' cognitive states
including their epistemic states and their motivational states.

The formula $[ \pi   ] \varphi$
has to be read ``$\varphi$ is true, according to the cognitive program $\pi$''.
As usual,
we define $\langle \pi \rangle$
to be the dual operator of $[ \pi   ] $, that is,
$\langle \pi \rangle \varphi=_{\mathit{def}} \neg [ \pi   ]\neg \varphi$.

The atomic program
$\equiv_i$
represents the standard
 S5,
partition-based and fully introspective notion of knowledge \cite{Fagin1995,Aumann99JGT}.
$[ \equiv_i ]\varphi$ has to be read ``$\varphi$ is true according to what agent $i$ knows''
or more simply
``agent $i$ knows that $\varphi$ is true'',
which just means that
``$\varphi$ is true in all worlds that agent $i$ envisages''.

The atomic programs
$ \preceq_{i,P} $
and 
$ \preceq_{i,D} $
capture, respectively,
agent $i$'s plausibility
ordering and agent $i$'s  desirability
ordering over facts.
In particular, 
$[ \preceq_{i,P} ] \varphi$
has to be read
 ``$\varphi$ is true at all states
 that, according to agent $i$,
 are at least as plausible as the current one'',
 while
 $[ \preceq_{i,D} ] \varphi$
has to be read
 ``$\varphi$ is true at all states
 that, according to agent $i$,
 are at least as desirable as the current one''.
 The atomic programs
$ \preceq_{i,P}^\sim$
and 
$ \preceq_{i,D}^\sim$
are the complements of
the atomic programs 
$ \preceq_{i,P}$
and 
$ \preceq_{i,D}$, respectively.
In particular, 
$[ \preceq_{i,P}^\sim ] \varphi$
has to be read
 ``$\varphi$ is true at all states
 that, according to agent $i$,
 are \emph{not} at least as plausible as the current one'',
 while
 $[ \preceq_{i,D}^\sim ] \varphi$
has to be read
 ``$\varphi$ is true at all states
 that, according to agent $i$,
 are \emph{not} at least as desirable as the current one''.
 The program
 constructs
 $ ;$,
 $ \cup$,
 $\cap$,
 $-$ and
$?$
are used to define complex cognitive programs
from
the atomic cognitive programs.
For example, 
the formula
 $[ \preceq_{i,P} \cup \preceq_{i,D} ] \varphi$
 has to be read 
  ``$\varphi$ is true at all states
 that, according to agent $i$,
 are either at least as
 plausible \emph{or}
 at least as
  desirable as the current one'',
 whereas
 the formula
  $[ \preceq_{i,P} \cap \preceq_{i,D} ] \varphi$
 has to be read 
  ``$\varphi$ is true at all states
 that, according to agent $i$,
 are  at least as
 plausible \emph{and}
 at least as
  desirable as the current one''.

The following definition
provides
truth conditions for
formulas
in $\mathcal{L}_{\mathrm{DLCA} }$:
\begin{definition}[Truth conditions]\label{truthcond}
Let $M = (  W,( \preceq_{i,P})_{i \in \mathit{Agt} }, (\preceq_{i,D})_{i \in \mathit{Agt} }  ,
(\equiv_i)_{i \in \mathit{Agt} } , V  )$ be a MCM
and let $w \in W$. Then:
\begin{eqnarray*}
M, w \models p & \Longleftrightarrow & p \in V( w) ,\\
M, w \models x & \Longleftrightarrow & x \in V( w), \\
M, w \models \neg \varphi & \Longleftrightarrow & M, w \not \models  \varphi, \\
M, w \models \varphi \wedge \psi & \Longleftrightarrow & M, w \models \varphi   \text{ and } M, w \models \psi ,\\
M, w \models [\pi] \varphi
 & \Longleftrightarrow & \forall
v \in W: \text{ if } w R_\pi v  \text{ then } M,v \models \varphi,
\end{eqnarray*}
where the binary relation $R_\pi$ on $W$ is inductively defined as follows,
with $\tau \in \{P,D\}$:
  \begin{align*}
       w  R_{\equiv_i}    v& \text{ iff }
w \equiv_i v , \\
    w R_{\preceq_{i,\tau}}  v  & \text{ iff }
w \preceq_{i,\tau} v,  \\
    w R_{\preceq_{i,\tau}^\sim}  v  & \text{ iff }
 w \equiv_i v \text{ and } w \not \preceq_{i,\tau } v  ,\\
    w  R_{ \pi ;  \pi'}v    & \text{ iff }
    \exists u \in W:
w R_{ \pi } u
 \text{ and }  uR_{  \pi'}  v,\\
    w R_{ \pi \cup  \pi'} v    & \text{ iff }
 w  R_{ \pi } v
 \text{ or }  w  R_{  \pi'}  v,\\
 w  R_{ \pi \cap  \pi'} v   & \text{ iff }
w R_{ \pi } v
 \text{ and }  w R_{  \pi'} v ,\\
  w  R_{ -\pi} v   & \text{ iff }
 v  R_{ \pi}w ,\\
 w R_{\varphi?} v   & \text{ iff }
w =v  \text{ and } M,w \models \varphi.
\end{align*}

 \end{definition}

For notational convenience,
we use $w R_\pi v$
and $(w,v) \in R_\pi$
as interchangeable notations.

We can build a variety 
of cognitive programs 
capturing different types
of plausibility
and desirability
relations between possible worlds.
For instance,
for every $\tau \in \{P,D\}$,
we can define:
  \begin{align*}
  \succeq_{i,\tau}    & =_{\mathit{def}}
-\preceq_{i,\tau},  \\
 \succ_{i,\tau}    & =_{\mathit{def}}
\succeq_{i,\tau}\cap\preceq_{i,\tau}^\sim ,\\
\succeq_{i,\tau}^\sim  & =_{\mathit{def}}
- \preceq_{i,\tau}^\sim,  \\
 \prec_{i,\tau}    & =_{\mathit{def}}
\preceq_{i,\tau} \cap\succeq_{i,\tau}^\sim , \\
 \approx_{i,\tau}   & =_{\mathit{def}}
\preceq_{i,\tau} \cap \succeq_{i,\tau}.
\end{align*}
The five definitions denote respectively
  ``at most as plausible (resp. desirable) as'',
    ``less plausible (resp. desirable) than'',
      ``not at most as plausible (resp. desirable) as'',
        ``more plausible (resp. desirable) than''
        and 
                ``equally plausible (resp. desirable) as''.

For every formula $\varphi$
in $\mathcal{L}_{\mathrm{DLCA} }$
we say
that $\varphi$
is valid, noted $\models_{\mathit{MCM}} \varphi$,
if and only if
for every multi-agent cognitive model
$M$
and world $w$
in $M$,
we have $M,w \models \varphi$.
Conversely,
we say that $\varphi$
is satisfiable if $\neg \varphi$
is not valid.

For a given multi-agent cognitive
model
$M = (  W,( \preceq_{i,P})_{i \in \mathit{Agt} }, (\preceq_{i,D})_{i \in \mathit{Agt} }  ,
(\equiv_i)_{i \in \mathit{Agt} }  ,
N, V  )$, we define $ || \varphi ||_M = \{v\in W: M,v \models \varphi\}$
to be
the truth set of $\varphi$
in $M$.
Moreover, for every $w \in W $
and for every $ i \in \mathit{Agt}$, we
define $
 || \varphi ||_{i,w,M} =   \{v\in W: M,v \models \varphi  \text{ and } w \equiv_i v\}$
to be
the truth set of $\varphi$
from $i$'s 
 point of view at state $w$ in $M$.

\section{Formalization of Cognitive Attitudes}\label{FormalizationPart}

In this section, 
we show how the logic $\mathrm{DLCA}$
can be used to model the variety
of cognitive attitudes of agents
that we have briefly discussed in the introduction.

\subsection{Epistemic Attitudes}

We start with the family
of epistemic attitudes by defining a
standard notion of belief.
We say that an agent 
believes
that $\varphi$
if and only if $\varphi$
is true at all states that the agent
considers maximally plausible.
\begin{definition}[Belief]\label{defBel}
Let $M = (  W,( \preceq_{i,P})_{i \in \mathit{Agt} }, (\preceq_{i,D})_{i \in \mathit{Agt} }  ,
(\equiv_i)_{i \in \mathit{Agt} }  ,
V  )$ be a MCM
and let $w \in W$.
We say that agent $i$ believes that $\varphi$
at $w$, noted $M,w \models \mathsf{B}_i \varphi  $, if and only if
$
\mathit{Best}_{i,P}(w)\subseteq || \varphi ||_M
$
where 
$
\mathit{Best}_{i,P}(w)= \{v \in W: w \equiv_i v \text{ and }
\forall u \in W, \text{ if } w \equiv_i u  \text{ then } u \preceq_{i,P} v  \}.
$
\end{definition}
As the following proposition highlights,
the previous notion of
belief is expressible
in the logic $\mathrm{DLCA}$
by means of the cognitive program $ \equiv_i ; 
 [\prec_{i,P}]\bot ?$.

\begin{proposition}
Let $M = (  W,( \preceq_{i,P})_{i \in \mathit{Agt} }, (\preceq_{i,D})_{i \in \mathit{Agt} }  ,
(\equiv_i)_{i \in \mathit{Agt} }  ,
 V  )$ be a MCM
and let $w \in W$.
Then,  we have
\begin{align*}
M,w \models \mathsf{B}_i \varphi   \text{ iff }
M,w \models \big[  \equiv_i ; 
 [\prec_{i,P}]\bot ? \big]  \varphi      .
\end{align*}
\end{proposition}

It is worth noting that the set 
$\mathit{Best}_{i,P}(w)$
in Definition \ref{defBel}
might be empty, since
it is not necessarily the case that
the relation $ \preceq_{i,P}$
is conversely well-founded.\footnote{This means that there could
be a world $v$
such that $w \equiv_i v$
and there is
a $\preceq_{i,P}$-infinite ascending chain from $v$.}
As a consequence, the belief operator 
$\mathsf{B}_i $
does not necessarily satisfy
Axiom D, i.e., the formula $\mathsf{B}_i \varphi \wedge \mathsf{B}_i \neg \varphi $
is satisfiable in the logic $\mathrm{DLCA}$. 

In the literature on epistemic
logic \cite{BaltagSmets},
mere belief of Definition \ref{defBel}
is usually distinguished from strong belief.
Specifically, 
we say that an agent 
strongly
believes
that $\varphi$
if and only if,
according to agent $i$,
all $\varphi$-worlds
are strictly
more plausible than all $\neg \varphi$-worlds.
\begin{definition}[Strong belief]\label{defSbel}
Let $M = (  W,( \preceq_{i,P})_{i \in \mathit{Agt} }, (\preceq_{i,D})_{i \in \mathit{Agt} }  ,
(\equiv_i)_{i \in \mathit{Agt} }  ,
V  )$ be a MCM
and let $w \in W$.
We say that agent $i$ strongly believes that $\varphi$
at $w$, noted $M,w \models \mathsf{SB}_i \varphi  $, if and only if
$
\forall v \in  || \varphi ||_{i,w,M}
\text{ and } \forall u \in  || \neg \varphi ||_{i,w,M} : u \prec_{i,P} v.
$

\end{definition}

As the following proposition highlights,
the previous notion of
strong 
belief is expressible
in the logic $\mathrm{DLCA}$
by means of the cognitive program $ \equiv_i ;  \varphi ? ;  \preceq_{i,P}
$.

\begin{proposition}
Let $M = (  W,( \preceq_{i,P})_{i \in \mathit{Agt} }, (\preceq_{i,D})_{i \in \mathit{Agt} }  ,
(\equiv_i)_{i \in \mathit{Agt} }  ,
V  )$ be a MCM
and let $w \in W$.
Then,  we have
\begin{align*}
M,w \models \mathsf{SB}_i \varphi   \text{ iff }
M,w \models \big[  \equiv_i ;  \varphi ? ;  \preceq_{i,P} \big]  \varphi   .
\end{align*}
\end{proposition}

Strong belief that $\varphi$ implies
belief that $\varphi$,
if the agent envisages at least
one state in which $\varphi$ is true.
This property is expressed by the following validity:
\begin{align}
\models_{\mathit{MCM}} & \big( \mathsf{SB}_i \varphi  \wedge \langle \equiv_i \rangle \varphi \big) \rightarrow \mathsf{B}_i \varphi 
\end{align}

Conditional belief is another notion
which has been studied by epistemic
logicians 
given its important role in belief
dynamics \cite{BenthemRevision}.
We say that an agent 
believes
that $\varphi$
conditional
on $\psi$,
or she would believe that $\varphi$
if she learnt that $\psi$,
if and only if,
according to the agent,
all most plausible $\psi$-worlds
are also $ \varphi$-worlds.

\begin{definition}[Conditional belief]\label{condSbel}
Let $M = (  W,( \preceq_{i,P})_{i \in \mathit{Agt} }, (\preceq_{i,D})_{i \in \mathit{Agt} }  ,
(\equiv_i)_{i \in \mathit{Agt} }  ,
V  )$ be a MCM
and let $w \in W$.
We say that agent $i$ would believe that $\varphi$
if she learnt that $\psi$
at $w$, noted $M,w \models \mathsf{B}_i (\psi,\varphi ) $, if and only if
$
\mathit{Best}_{i,P}(\psi, w)\subseteq || \varphi ||_M
$,
where 
$
\mathit{Best}_{i,P}(\psi, w)= \{v \in  || \psi ||_{i,w,M}
: 
\forall u \in  || \psi ||_{i,w,M},  u \preceq_{i,P} v  \}.
$
\end{definition}
Note that $\mathit{Best}_{i,P}(\top, w) = \mathit{Best}_{i,P}( w)$.

As for belief
and strong belief,
we have a specific cognitive
program
 $ \equiv_i ; 
( \psi \wedge [\prec_{i,P}]\neg \psi )?$
corresponding to
the belief
that $\varphi$
conditional on $\psi$,
so that the latter can be represented in 
in the language of the logic $\mathrm{DLCA}$.

\begin{proposition}
Let $M = (  W,( \preceq_{i,P})_{i \in \mathit{Agt} }, (\preceq_{i,D})_{i \in \mathit{Agt} }  ,
(\equiv_i)_{i \in \mathit{Agt} }  ,
 V  )$ be a MCM
and let $w \in W$.
Then,  we have
\begin{align*}
M,w \models  \mathsf{B}_i (\psi,\varphi )   \text{ iff }
M,w \models \big[  \equiv_i ; 
( \psi \wedge [\prec_{i,P}]\neg \psi )? \big]  \varphi  .    
\end{align*}
\end{proposition}

 \subsection{Motivational Attitudes I: Desires}

 The first kind of motivational
 attitude we consider is desire.
Following \cite{DuboisLorini},
we say that 
 an agent 
desires
that $\varphi$
if and only if 
all states that
the agent envisages
at which 
$\varphi$
is true are not minimally desirable
for  her.
In other words,
desiring that $\varphi$
consists in having some degree of attraction for
all situations in which
 $\varphi$ is true,
 since minimally desirable states are those to which the agent
 is not attracted at all.
 
\begin{definition}[Desire]\label{defDes}
Let $M = (  W,( \preceq_{i,P})_{i \in \mathit{Agt} }, (\preceq_{i,D})_{i \in \mathit{Agt} }  ,
(\equiv_i)_{i \in \mathit{Agt} }  ,
V  )$ be a MCM
and let $w \in W$.
We say that agent $i$ desires that $\varphi$
at $w$, noted $M,w \models \mathsf{D}_i \varphi  $, if and only if
$
\mathit{Worst}_{i,D}(w)\cap || \varphi ||_M = \emptyset
$,
where 
$
\mathit{Worst}_{i,D}(w)= \{v \in W: w \equiv_i v \text{ and }
\forall u \in W, \text{ if } w \equiv_i u  \text{ then } v \preceq_{i,D} u  \}.
$
\end{definition}
As the following proposition highlights,
the previous notion of
desire
is characterized by
the cognitive program
 $ \equiv_i ; 
 [\succ_{i,D}]\bot ?$.
\begin{proposition}
Let $M = (  W,( \preceq_{i,P})_{i \in \mathit{Agt} }, (\preceq_{i,D})_{i \in \mathit{Agt} }  ,
(\equiv_i)_{i \in \mathit{Agt} }  ,
V  )$ be a MCM
and let $w \in W$.
Then,  we have
\begin{align*}
M,w \models \mathsf{D}_i \varphi   \text{ iff }
M,w \models \big[  \equiv_i ; 
 [\succ_{i,D}]\bot ? \big] \neg \varphi      .
\end{align*}
\end{proposition}

Similarly to the set $\mathit{Best}_{i,P}(w)$
in Definition \ref{defBel},
the set 
$\mathit{Worst}_{i,D}(w)$
in Definition \ref{defDes}
might be empty, since
it is not necessarily the case that
the relation $ \preceq_{i,D}$
is  well-founded.\footnote{This means that there could
be a world $v$
such that $w \equiv_i v$
and there is
a $\preceq_{i,D}$-infinite descending chain from $v$.}
As a consequence,  desires 
are not necessarily consistent
and an agent may desire the tautology, i.e.,
the formulas
$\mathsf{D}_i \varphi \wedge \mathsf{D}_i \neg \varphi $
and
 $\mathsf{D}_i \top $
are satisfiable in the logic $\mathrm{DLCA}$. 
As emphasized by  \cite{DuboisLorini},
this notion of desire satisfies
the following property:
\begin{align}
\models_{\mathit{MCM}} & \mathsf{D}_i \varphi
\rightarrow \mathsf{D}_i ( \varphi \wedge \psi)
\end{align}
Indeed, if an agent
has  some degree of attraction for
all situations in which
 $\varphi$ is true then, clearly,
 it should have some degree
 of attraction for all situations
 in which  $\varphi \wedge \psi$ is true,
 since all $\varphi \wedge \psi$-situations
 are also
 $\varphi $-situations.

Note that there
is no counterpart
of this property
for belief, 
as the formula
$\mathsf{B}_i \varphi
\rightarrow \mathsf{B}_i ( \varphi \wedge \psi)$
is clearly not valid.\footnote{See
\cite{DuboisLorini} for more details
about the differences between
the notion of belief and the notion of desire.}

 It is a property
 that the
 notion
 of desire shares
 with the \emph{open reading}
 of the concept of
permission
 studied in the area of deontic logic (see, e.g., \cite{Anglberger,LewisPermission}).\footnote{
 According to deontic logicians, there are at least two candidate readings of the statement
  ``$\varphi$
  is permitted'': (i) every instance of $\varphi$ is OK according to the normative regulation,
  and (ii) at least one instance of $\varphi$ (but possibly not all) is OK according to the
normative regulation. The former is the so-called \emph{open reading} of permission.
 }
 One way of blocking this inference
 is by strengthening the notion
 of desire.
 We  say that an agent 
strongly
desires
that $\varphi$
if and only if,
according to agent $i$,
all $\varphi$-worlds
are strictly
more desirable than all $\neg \varphi$-worlds.
\begin{definition}[Strong desire]\label{defSbel}
Let $M = (  W,( \preceq_{i,P})_{i \in \mathit{Agt} }, (\preceq_{i,D})_{i \in \mathit{Agt} }  ,
(\equiv_i)_{i \in \mathit{Agt} }  ,
 V  )$ be a MCM
and let $w \in W$.
We say that agent $i$ strongly desires that $\varphi$
at $w$, noted $M,w \models \mathsf{SD}_i \varphi  $, if and only if
$
\forall v \in  || \varphi ||_{i,w,M}
\text{ and } \forall u \in  || \neg \varphi ||_{i,w,M} : u \prec_{i,D} v.
$

\end{definition}
As for desire,
there exists
a cognitive program
which characterizes
strong desire, namely,
the program $  \equiv_i ;  \varphi ? ;  \preceq_{i,D}$.

\begin{proposition}
Let $M = (  W,( \preceq_{i,P})_{i \in \mathit{Agt} }, (\preceq_{i,D})_{i \in \mathit{Agt} }  ,
(\equiv_i)_{i \in \mathit{Agt} }  ,
V  )$ be a MCM
and let $w \in W$.
Then,  we have
\begin{align*}
M,w \models \mathsf{SD}_i \varphi   \text{ iff }
M,w \models \big[  \equiv_i ;  \varphi ? ;  \preceq_{i,D} \big]  \varphi      .
\end{align*}
\end{proposition}
We have that strong desire implies desire,
when the agent envisages at least one state in which $\varphi$
is false:
\begin{align}
\models_{\mathit{MCM}} & (  \mathsf{SD}_i \varphi  \wedge \langle \equiv_i \rangle \neg \varphi )
\rightarrow \mathsf{D}_i \varphi 
\end{align}
Unlike desire,
it is not necessarily the case
that strongly desiring that $\varphi$
implies
strongly desiring  that $\varphi \wedge \psi$,
i.e., $\mathsf{SD}_i \varphi   \wedge \neg \mathsf{SD}_i (\varphi \wedge \psi)  $ is satisfiable  in the logic $\mathrm{DLCA}$.
Indeed,
 strongly  desiring that $\varphi$
is compatible with 
envisaging a situation in which 
$\varphi \wedge \psi$
holds
and another situation
in which $\varphi \wedge \neg \psi$
holds
such that the first situation
is less desirable than the second.

The last motivational attitude we consider
is conditional desire which parallels the notion
of conditional belief of Definition \ref{condSbel}.
We say that an agent 
desires
that $\varphi$
conditional
on $\psi$,
or she would desire that $\varphi$
if she started to desire that $\psi$,
if and only if,
according to agent $i$,
there is no least desirable $\neg \psi$-world
which is also 
a  $\varphi$-world.
The idea behind this notion is the following. 
If the agent started to desire that $\psi$,
all $\psi$-worlds
would start to have some degree of attraction for her
and the least desirable  $\neg \psi$-worlds would become the minimally desirable worlds.
Therefore, the fact that there  is no least desirable 
 $\neg \psi$-world
which is also 
a  $\varphi$-world guarantees that, if the agent started to desire that $\psi$, 
no $\varphi $-world would be included in the set of 
minimally desirable worlds
for the agent. The latter means that,
if the agent started to desire that $\psi$, 
all $\varphi $-worlds would have
some degree of attraction for her
and she would desire that $\varphi$.

\begin{definition}[Conditional desire]\label{condDesire}
Let $M = (  W,( \preceq_{i,P})_{i \in \mathit{Agt} }, (\preceq_{i,D})_{i \in \mathit{Agt} }  ,
(\equiv_i)_{i \in \mathit{Agt} }  ,
V  )$ be a MCM
and let $w \in W$.
We say that agent $i$ would desire that $\varphi$
if she started to desire that $\psi$
at $w$, noted $M,w \models \mathsf{D}_i (\psi,\varphi ) $, if and only if
$
\mathit{Worst}_{i,D}(\neg \psi, w)\cap || \varphi ||_M= \emptyset
$,
with
$
\mathit{Worst}_{i,D}(\neg\psi, w)= \{v \in  ||  \neg \psi ||_{i,w,M}
: 
\forall u \in  || \neg \psi ||_{i,w,M},  v \preceq_{i,D} u  \}.
$
\end{definition}

As for the other cognitive attitudes,
there is a specific  cognitive
program
which characterizes conditional desire.

\begin{proposition}
Let $M = (  W,( \preceq_{i,P})_{i \in \mathit{Agt} }, (\preceq_{i,D})_{i \in \mathit{Agt} }  ,
(\equiv_i)_{i \in \mathit{Agt} }  ,
 V  )$ be a MCM
and let $w \in W$.
Then,  we have
\begin{align*}
M,w \models \mathsf{D}_i (\psi,\varphi )  \text{ iff }
M,w \models \big[  \equiv_i ; 
( \neg \psi \wedge [\succ_{i,D}] \psi )? \big] \neg \varphi  .    
\end{align*}
\end{proposition}

 In 
Section \ref{LogicPart},
we   emphasized
that
 the  relations
 $\preceq_{i,P}$
 and $\preceq_{i,D}$
 do not interact since
our logic 
is aimed at modeling
 ideal
 rational 
 agents 
 with no wishful thinking
 and, more generally, 
 with no cognitive biases.
 We conclude this 
 section by showing how
 the assumption of independence
 between
 epistemic plausibility and
desirability could be relaxed
 and, consequently, how wishful
 thinking could be modeled
 in our framework.
 
 A wishful thinker is nothing
 but an agent who systematically
 believes
 what she strongly desires
 in the absence of a reason to believe
 the contrary.
 Such a connection between
 the agent's beliefs
 and desires is captured
 by the following
 ``wishful thinking'' (WT)
  constraint on
 MCMs:
  \begin{align*}
\forall w \in W: \mathit{Best}_{i,P}(w) \subseteq 
   \mathit{Best}_{i,D}(w) \text{ or }
     \mathit{Best}_{i,P}(w) \subseteq 
  \mathit{Worst}_{i,D}(w),
  \end{align*}
  where 
  $\mathit{Best}_{i,P}(w)$
  and
  $\mathit{Worst}_{i,D}(w)$
  are defined as in Definitions \ref{defBel}
  and \ref{defDes},
  and $\mathit{Best}_{i,D}(w)= \{v \in W: w \equiv_i v \text{ and }
\forall u \in W, \text{ if } w \equiv_i u  \text{ then } u \preceq_{i,D} v  \}$.
It is routine
to verify that if the MCM
$M = (  W,( \preceq_{i,P})_{i \in \mathit{Agt} }, (\preceq_{i,D})_{i \in \mathit{Agt} }  ,
(\equiv_i)_{i \in \mathit{Agt} }  ,
 V  )$
satisfies the previous constraint WT,
then the following holds for every $w \in W$:
  \begin{align*}
 M,w \models 
( \mathsf{SD}_i \varphi \wedge \neg \mathsf{B}_i \neg
\varphi
) \rightarrow \mathsf{B}_i \varphi.
  \end{align*}
  We leave for future work
  an 
  in-depth analysis
 of the variant of our logic
  in which 
  wishful thinking
is enabled.

\subsection{Motivational Attitudes II: Preferences }\label{choiceSect}

We consider two views about  comparative statements between
formulas
of the form  
  ``agent $i$ prefers $\varphi $ to $\psi$''
  or  ``the state of affairs 
 $\varphi$  is for agent $i$ at least as good as 
 the state of affairs $\psi$''.
 According to the optimistic view, 
 when assessing
 whether $\varphi$
 is at least as good as $\psi$,
 an agent focuses on the best $\varphi$-situations 
 in comparison with  the best $\psi$-situations.
 Specifically,
 an  ``optimistic'' agent $i$ prefers $\varphi$
 to $\psi$
 if and only if,
 for every $\psi$-situation
 envisaged by $i$
 there exists a $\varphi$-situation
 envisaged by $i$ such that
 the latter is at least as desirable
 as the former. 
 According to the pessimistic view, 
 she  focuses on the worst $\varphi$-situations 
 in comparison with  the worst $\psi$-situations.
  Specifically,
   a  ``pessimistic'' 
   agent $i$ prefers $\varphi$
 to $\psi$
 if and only if,
 for every $\varphi$-situation
 envisaged by $i$
 there exists a $\psi$-situation
 envisaged by $i$ such that
 the former is at least as desirable
 as the latter.

Let us first define 
a dyadic operator for preference
according to the optimistic view.
\begin{definition}[Preference: optimistic view]\label{CompDesO}
Let $M = (  W,( \preceq_{i,P})_{i \in \mathit{Agt} },  (\preceq_{i,D})_{i \in \mathit{Agt} }  ,
(\equiv_i)_{i \in \mathit{Agt} }  ,
V  )$ be a MCM
and let $w \in W$.
We say that, according to agent $i$'s optimistic assessment,
$\varphi$
is at least as good as $\psi$
at $w$, noted $M,w \models \mathsf{P}_i^{\mathit{Opt}} (\psi \preceq \varphi  )$, if and only if
$
\forall u \in  || \psi ||_{i,w,M},
\exists v \in  || \varphi ||_{i,w,M} : u \preceq_{i,D} v.
$

\end{definition}

As the following proposition highlights,
it
is expressible
in the language $\mathcal{L}_{\mathrm{DLCA} }$.

\begin{proposition}
Let $M = (  W,( \preceq_{i,P})_{i \in \mathit{Agt} }, (\preceq_{i,D})_{i \in \mathit{Agt} }  ,
(\equiv_i)_{i \in \mathit{Agt} }  ,
V  )$ be a MCM
and let $w \in W$.
Then,  we have
\begin{align*}
M,w \models\mathsf{P}_i^{\mathit{Opt}} (\psi \preceq \varphi  )   \text{ iff }
M,w \models \big[  \equiv_i ; \psi ? \big]   \langle \preceq_{i,D}\rangle \varphi .
\end{align*}
\end{proposition}

Let us now define
preference
according to the pessimistic view.
\begin{definition}[Preference: pessimistic view]\label{CompDesP}
Let $M = (  W,( \preceq_{i,P})_{i \in \mathit{Agt} }, (\preceq_{i,D})_{i \in \mathit{Agt} }  ,
(\equiv_i)_{i \in \mathit{Agt} }  ,
 V  )$ be a MCM
and let $w \in W$.
We say that, according to agent $i$'s pessimistic assessment,
$\varphi$
is at least as good as $\psi$
at $w$, noted $M,w \models \mathsf{P}_i^{\mathit{Pess}}(\psi \preceq \varphi  )$, if and only if
$
\forall v \in  || \varphi ||_{i,w,M},
\exists u \in  || \psi ||_{i,w,M} : u \preceq_{i,D} v.
$

\end{definition}

As for the optimistic view,
the
pessimistic view 
is also expressible
in the language $\mathcal{L}_{\mathrm{DLCA} }$.

\begin{proposition}
Let $M = (  W,( \preceq_{i,P})_{i \in \mathit{Agt} }, (\preceq_{i,D})_{i \in \mathit{Agt} }  ,
(\equiv_i)_{i \in \mathit{Agt} }  ,
V  )$ be a MCM
and let $w \in W$.
Then,  we have
\begin{align*}
M,w \models\mathsf{P}_i^{\mathit{Pess}} (\psi \preceq \varphi  )   \text{ iff }
M,w \models \big[  \equiv_i ; \varphi ? \big]   \langle \succeq_{i,D}\rangle \psi.
\end{align*}
\end{proposition}

Thanks to the totality of the relation $ \preceq_{i,D}$  (Constraint C2 in Definition \ref{DefSemanticsModel}),
dyadic preference over formulas is total too. This fact is illustrated by the following validity.
For every $x \in \{\mathit{Opt}, \mathit{Pess}\}$:
\begin{align}
\models_{\mathit{MCM}} &   \mathsf{P}_i^{x} (\psi \preceq \varphi  )  \vee  \mathsf{P}_i^{x} (\varphi \preceq \psi  ) 
\end{align}
To see this suppose $M,w \models  \neg  \mathsf{P}_i^{\mathit{Opt}} (\psi \preceq \varphi  ) $
for an arbitrary model $M$
and world $w$ in $M$. Because of Constraint C2 in Definition \ref{DefSemanticsModel},
the latter implies that
$
\exists u \in  || \psi ||_{i,w,M},
\forall v \in  || \varphi ||_{i,w,M} : v \prec_{i,D} u
$. Therefore, $
\forall v \in  || \varphi ||_{i,w,M},
\exists u \in  || \psi ||_{i,w,M} : v \preceq_{i,D} u
$ which is equivalent to $M,w \models    \mathsf{P}_i^{\mathit{Opt}} (\varphi \preceq \psi  ) $.
The case $x= \mathit{Pess}$
can be proved in an analogous way.

The previous notion of (optimistic and pessimistic) preference
does not depend on what the agent believes. This means
that, in order to assess whether $\varphi$
is at least as good as $\psi$,
an agent also takes into account worlds
that are implausible (or, more generally, 
not maximally plausible). 
\emph{Realistic}
preference
requires that an agent
compares two formulas $\varphi$
and $\psi$
only with respect to the set of most plausible states.
This idea has been discussed
in the area of qualitative decision theory
by different authors \cite{Boutilier94,BrafmanTennen,BrafmanTennen2}.

%
%

%

The following definition
introduces 
\emph{realistic} preference according to the optimistic view.

\begin{definition}[Realistic preference: optimistic view]\label{RChoiceO}
Let $M = (  W, ( \preceq_{i,P})_{i \in \mathit{Agt} }, (\preceq_{i,D})_{i \in \mathit{Agt} }  ,
(\equiv_i)_{i \in \mathit{Agt} }  ,
V  )$ be a MCM
and let $w \in W$.
We say that, according to agent $i$'s optimistic assessment,
$\varphi$
is realistically at least as good as $\psi$
at $w$, noted $M,w \models \mathsf{RP}_i^{\mathit{Opt}} (\psi \preceq \varphi  )$, if and only if
$
\forall u \in  \mathit{Best}_{i,P}(w) \cap || \psi ||_{i,w,M},
\exists v \in  \mathit{Best}_{i,P}(w)  \cap || \varphi ||_{i,w,M} : u \preceq_{i,D} v.
$

\end{definition}

The idea is that 
 an  ``optimistic'' agent $i$ considers
 $\varphi$ \emph{realistically}
 at least as good as $\psi$
 if and only if,
 for every $\psi$-situation
in agent $i$'s belief set
 there exists a $\varphi$-situation
in agent $i$'s belief set such that
 the latter is at least as good
 as the former.

The previous notion as well
is expressible
in the language $\mathcal{L}_{\mathrm{DLCA} }$.
\begin{proposition}
Let $M = (  W,( \preceq_{i,P})_{i \in \mathit{Agt} }, (\preceq_{i,D})_{i \in \mathit{Agt} }  ,
(\equiv_i)_{i \in \mathit{Agt} }  ,
V  )$ be a MCM
and let $w \in W$.
Then,  we have
\begin{align*}
M,w \models\mathsf{RP}_i^{\mathit{Opt}} (\psi \preceq \varphi  )   \text{ iff } &
M,w \models \big[  
\equiv_i ; 
 [\prec_{i,P}]\bot ? ; \psi ? \big]   \langle \preceq_{i,D} \cap
 (\equiv_i ; 
 [\prec_{i,P}]\bot ?)
 \rangle \varphi.
\end{align*}
\end{proposition}

 The following definition
introduces 
\emph{realistic} preference according to the pessimistic view.
\begin{definition}[Realistic preference: pessimistic view]\label{RChoiceP}
Let $M = (  W, ( \preceq_{i,P})_{i \in \mathit{Agt} }, (\preceq_{i,D})_{i \in \mathit{Agt} }  ,
(\equiv_i)_{i \in \mathit{Agt} }  ,
 V  )$ be a MCM
and let $w \in W$.
We say that, according to agent $i$'s pessimistic assessment,
$\varphi$
is realistically at least as good as $\psi$
at $w$, noted $M,w \models \mathsf{RP}_i^{\mathit{Pess}}(\psi \preceq \varphi  )$, if and only if
$
\forall v \in  \mathit{Best}_{i,P}(w) \cap || \varphi ||_{i,w,M},
\exists u \in  \mathit{Best}_{i,P}(w) \cap || \psi ||_{i,w,M}: u \preceq_{i,D} v.
$
\end{definition}
The idea is that 
 a  ``pessimistic'' agent $i$ considers
 $\varphi$ \emph{realistically}
 at least as good as $\psi$
 if and only if,
 for every $\varphi$-situation
in agent $i$'s belief set
 there exists a $\psi$-situation
in agent $i$'s belief set such that
 the former is at least as good
 as the latter. 

It is
also expressible
in the language $\mathcal{L}_{\mathrm{DLCA} }$.
\begin{proposition}
Let $M = (  W,( \preceq_{i,P})_{i \in \mathit{Agt} }, (\preceq_{i,D})_{i \in \mathit{Agt} }  ,
(\equiv_i)_{i \in \mathit{Agt} }  ,
V  )$ be a MCM
and let $w \in W$.
Then,  we have
\begin{align*}
M,w \models\mathsf{RP}_i^{\mathit{Pess}} (\psi \preceq \varphi  )   \text{ iff } &
M,w \models \big[  
\equiv_i ; 
 [\prec_{i,P}]\bot ? ; \varphi ? \big]   \langle \succeq_{i,D} \cap
 (\equiv_i ; 
 [\prec_{i,P}]\bot ?)
 \rangle \psi.
\end{align*}
\end{proposition}

Like dyadic preference over formulas,
realistic dyadic preference over formulas is total.
In fact, for every
for $x \in \{\mathit{Opt}, \mathit{Pess}\}$, we have:
\begin{align}
\models_{\mathit{MCM}} &   \mathsf{RP}_i^{x} (\psi \preceq \varphi  )  \vee  \mathsf{RP}_i^{x} (\varphi \preceq \psi  ) 
\end{align}

The following abbreviations
define \emph{strict}
variants of dyadic preference
operators:
  \begin{align*}
 &  \mathsf{P}_i^{\mathit{Opt}} (\psi \prec \varphi  )     =_{\mathit{def}}
 \neg  \mathsf{P}_i^{\mathit{Opt}} (\varphi \preceq \psi  )\\
    &\mathsf{P}_i^{\mathit{Pess}} (\psi \prec \varphi  )     =_{\mathit{def}}
 \neg  \mathsf{P}_i^{\mathit{Pess}}(\varphi \preceq \psi  )\\
    &\mathsf{RP}_i^{\mathit{Opt}} (\psi \prec \varphi  )     =_{\mathit{def}}
 \neg  \mathsf{RP}_i^{\mathit{Opt}} (\varphi \preceq \psi  )\\
    &   \mathsf{RP}_i^{\mathit{Pess}} (\psi \prec \varphi  )     =_{\mathit{def}}
 \neg  \mathsf{RP}_i^{\mathit{Pess}} (\varphi \preceq \psi  )
\end{align*}
$   \mathsf{P}_i^{\mathit{Opt}} (\psi \prec \varphi  ) $ (resp. 
$   \mathsf{P}_i^{\mathit{Pess}} (\psi \prec \varphi  ) $)
has to be read ``according to $i$'s optimistic (resp. pessimistic) assessment, 
 $\varphi$  is better than
 $\psi$''.
$   \mathsf{RP}_i^{\mathit{Opt}} (\psi \prec \varphi  )$
(resp. $   \mathsf{RP}_i^{\mathit{Pess}} (\psi \prec \varphi  )$)
has to be read 
``according to agent $i$'s optimistic (resp. pessimistic) assessment,
$\varphi$
is realistically better than $\psi$''.

We conclude this section by defining two notions
of monadic preference 
and  corresponding  two notions of \emph{realistic}
monadic preference, respectively noted
$   \mathsf{P}_i^{\mathit{Opt}} \varphi$,
$   \mathsf{P}_i^{\mathit{Pess}} \varphi$,
$   \mathsf{RP}_i^{\mathit{Opt}} \varphi$ and $   \mathsf{RP}_i^{\mathit{Pess}} \varphi$:
  \begin{align*}
&     \mathsf{P}_i^{\mathit{Opt}} \varphi       =_{\mathit{def}}    \mathsf{P}_i^{\mathit{Opt}} (\neg \varphi \prec \varphi  )\\
 &  \mathsf{P}_i^{\mathit{Pess}} \varphi      =_{\mathit{def}}    \mathsf{P}_i^{\mathit{Pess}} (\neg \varphi \prec \varphi  )\\
  & \mathsf{RP}_i^{\mathit{Opt}} \varphi      =_{\mathit{def}}    \mathsf{RP}_i^{\mathit{Opt}} (\neg \varphi \prec \varphi  )\\
   &\mathsf{RP}_i^{\mathit{Pess}} \varphi      =_{\mathit{def}}    \mathsf{RP}_i^{\mathit{Pess}} (\neg \varphi \prec \varphi  )
\end{align*}
An optimistic (resp. pessimistic) agent has a preference for $\varphi$, noted $\mathsf{P}_i^{\mathit{Opt}} \varphi$ (resp.  $\mathsf{P}_i^{\mathit{Pess}} \varphi$),
if and only if, 
according to her optimistic (resp. pessimistic) assessment,
$\varphi$
is better than $\neg \varphi$.
An optimistic (resp. pessimistic) agent has a realistic preference for $\varphi$,
noted $\mathsf{RP}_i^{\mathit{Opt}} \varphi$ (resp.  $\mathsf{RP}_i^{\mathit{Pess}} \varphi$),
if and only if, 
according to her optimistic (resp. pessimistic) assessment,
$\varphi$
is
realistically better than $\neg \varphi$.

The following validity illustrates the relationship between the notion of desire defined in Definition \ref{defDes} and the previous notion of pessimistic 
monadic
preference: 
\begin{align}
\models_{\mathit{MCM}} &  \neg \mathsf{D}_i \top \rightarrow (    \mathsf{D}_i \varphi  \leftrightarrow \mathsf{P}_i^{\mathit{Pess}} \varphi)
\end{align}
This means that if there exists at least a minimally desirable state
for agent $i$ (condition $\neg \mathsf{D}_i \top$),
then $i$  desires that $\varphi $
if and only if, 
according to her pessimistic assessment,
$\varphi$
is
better than $\neg \varphi$.

\subsection{Example }\label{example}

In the previous sections, 
we have defined a variety
of cognitive attitudes
of epistemic
and motivational type.
Let us illustrate them with the help
of the crossroad scenario
sketched in the introduction.
For simplicitly,
we assume
that $\mathit{Agt}= \{1,2\}$
and that the set of atomic
propositions 
$\mathit{Atm}$
includes the following elements with their corresponding
meaning: $\mathit{co}$ (``agent $1$ and agent $2$ collide''),
$\mathit{lo}_1$ (``agent $1$ loses  time'')
and 
$\mathit{lo}_2$ (``agent $2$ loses time'').

We are going to make different hypotheses
about the agents' cognitive attitudes
and present a number of conclusions
that can be drawn from them.
Our initial hypothesis
concerns the agents' knowledge:
\begin{align*}
\varphi_1=_{\mathit{def}} &\bigwedge_{i \in \{1,2 \}}
[ \equiv_i ] \Big(     
\big( ( \mathit{lo}_1 \wedge \neg \mathit{lo}_2 ) \rightarrow
\neg \mathit{co}   \big) \wedge \\
&\big( ( \neg \mathit{lo}_1 \wedge  \mathit{lo}_2 ) \rightarrow
\neg \mathit{co}   \big) \wedge 
\neg 
(\neg  \mathit{lo}_1 \wedge  \neg \mathit{lo}_2 )
   \Big).
\end{align*}
According to hypothesis 
$\varphi_1$,
agents $1$
and $2$
know 
(i) that there will be no collision if one of them
loses time
by letting the other pass,
and (ii) that  necessarily 
one of them will lose time
(since if they both pass,
there will be a collision
so that they will both
lose time).

Our second hypothesis
concerns what the agents
merely
envisage:
\begin{align*}
\varphi_2=_{\mathit{def}} &\bigwedge_{i \in \{1,2 \}}
\Big( \langle \equiv_i \rangle 
\mathit{co} \wedge 
\langle \equiv_i \rangle
(\mathit{lo}_1 \wedge \neg \mathit{lo}_2 )
\wedge 
\langle \equiv_i \rangle
(\neg \mathit{lo}_1 \wedge \mathit{lo}_2 )
\wedge 
\langle \equiv_i \rangle
(\mathit{lo}_1 \wedge \mathit{lo}_2 \wedge 
\neg \mathit{co}  )
\Big)
.
\end{align*}
According to hypothesis 
$\varphi_2$,
agents $1$ and $2$
envisage four possible
situations:
(i) the situations in which they collide,
(ii) the two situations in which
one of them
loses its time
while the other does not,
and
(iii) the situation in which they both lose time
because
of a collision.


We conclude with the following hypothesis about the agents' motivations,
according to which
each agent
strongly desires not
to collide and 
strongly desires
not to lose time:
\begin{align*}
\varphi_3=_{\mathit{def}} &
\bigwedge_{i \in \{1,2 \}}
\big(\mathsf{SD}_i\neg \mathit{lo}_i 
\wedge 
\mathsf{SD}_i\neg \mathit{co} \big) .
\end{align*}

As
the following validities highlight,
the previous hypotheses
lead to different conclusions
about the agents' epistemic
and motivational attitudes:
\begin{align}
\models_{\mathit{MCM}} &  
\varphi_1 
\rightarrow  \bigwedge_{i \in \{1,2 \}}
\Big( [ \equiv_i ]
\big( \mathit{co}  \rightarrow
( \mathit{lo}_1 \wedge \mathit{lo}_2 ) \big)
\wedge
\mathsf{B}_i ( \neg \mathit{lo}_1, \mathit{lo}_2 )
\wedge \mathsf{B}_i ( \neg \mathit{lo}_2, \mathit{lo}_1 )
\Big)
\\
 \models_{\mathit{MCM}} &  
( \varphi_2 \wedge \varphi_3)
\rightarrow \bigwedge_{i \in \{1,2 \}}
\big(
\mathsf{SD}_i ( 
\neg \mathit{lo}_i \wedge \neg \mathit{co}
 )\wedge 
 \mathsf{D}_i \neg  \mathit{lo}_i \wedge \mathsf{D}_i \neg  \mathit{co}
\big) \\
 \models_{\mathit{MCM}} &
(\varphi_1 \wedge \varphi_2 \wedge \varphi_3)
\rightarrow \bigwedge_{i \in \{1,2 \}}
\big(\mathsf{D}_i (\mathit{lo}_1 \wedge \neg
   \mathit{lo}_2)
   \wedge 
     \mathsf{D}_i (\neg \mathit{lo}_1 \wedge 
   \mathit{lo}_2)
 \big)\\
 \models_{\mathit{MCM}} &
( \varphi_2 \wedge \varphi_3)
\rightarrow \bigwedge_{i \in \{1,2 \}}
\big(\neg \mathsf{SD}_i (\mathit{lo}_1 \wedge \neg
   \mathit{lo}_2)
   \wedge  \neg 
     \mathsf{SD}_i (\neg \mathit{lo}_1 \wedge 
   \mathit{lo}_2)
 \big)
\end{align}
 The single hypothesis 
$\varphi_1$ leads
to the conclusion 
(i) that the agents
know that a collision implies
that they both lose time,
and (ii) that they  believe
that an agent loses time
conditional on the fact that
the other does not.
Thanks to
the set of hypotheses $\{  \varphi_2,
\varphi_3\}$,
we can conclude
(i) that 
each agent strongly desires
 not
 to lose
 time
and 
 to avoid a collision,
 and
 (ii) each agent
 has both the desire
 not to lose
 time and the desire
 to avoid a collision.
 Finally, thanks
 to the set of hypotheses
  $\{ \varphi_1, \varphi_2,
\varphi_3\}$,
we can conclude that each agent
finds desirable
the situations in which
only one of them loses time by letting the other pass.
As the last validity indicates,
 such  situations are merely desirable
for the agent but not strongly desirable.

\section{Axiomatization }\label{AxiomatizationPart}

In this section, we provide a sound and complete
axiomatization
for the Dynamic Logic of Cognitive Attitudes
($\mathrm{DLCA}$).
The first step consists in precisely
defining this logic
which includes
several axioms
and rule
of necessitation 
for the modalities
$[\pi]$
as well as one non-standard
rule
of inference for nominals.

\begin{definition}[Logic $\mathrm{DLCA}$]\label{axiomatics}
We define  $\mathrm{DLCA}$
to be the extension of classical
propositional logic given by the following
 axioms and rules with $\tau \in \{P,D\}$:
\begin{align}
& ( [\pi] \varphi
\wedge [\pi] (\varphi \rightarrow \psi))
\rightarrow [\pi] \psi
 \tagLabel{K$_{\pi}$}{ax:Kequiv}\\
 & [\equiv_i] \varphi
\rightarrow  \varphi
 \tagLabel{T$_{\equiv_i}$}{ax:Tequiv}\\
  & [\equiv_i] \varphi
\rightarrow  [\equiv_i] [\equiv_i] \varphi
 \tagLabel{4$_{\equiv_i}$}{ax:4equiv}\\
   & \neg [\equiv_i] \varphi
\rightarrow  [\equiv_i] \neg [\equiv_i] \varphi
 \tagLabel{5$_{\equiv_i}$}{ax:5equiv}\\
 &  [\preceq_{i,\tau}] \varphi  \rightarrow \varphi
 \tagLabel{T$_{\preceq_{i,\tau}}$}{ax:Tpreceq}\\
  & [\preceq_{i,\tau}] \varphi
\rightarrow  [\preceq_{i,\tau}] [\preceq_{i,\tau}] \varphi
 \tagLabel{4$_{\preceq_{i,\tau}}$}{ax:4preceq}\\
   &  [\equiv_i ] \varphi
\rightarrow  [\preceq_{i,\tau} ] \varphi
 \tagLabel{Inc$_{\preceq_{i,\tau},\equiv_i}$}{ax:Int}\\
   &  \big( \langle \equiv_i \rangle \varphi \wedge  \langle \equiv_i \rangle \psi\big)
\rightarrow \big(  \langle \equiv_i \rangle (\varphi \wedge \langle \preceq_{i,\tau} \rangle \psi ) \vee
 \langle \equiv_i \rangle (\psi \wedge \langle \preceq_{i,\tau} \rangle \varphi ) \big)
 \tagLabel{Conn$_{\preceq_{i,\tau},\equiv_i}$}{ax:Conn}\\
   &  [\pi {;} \pi' ] \varphi
\leftrightarrow
[\pi  ] [\pi' ] \varphi
 \tagLabel{Red$_{;}$}{ax:Prg1}\\
    &  [\pi \cup \pi' ] \varphi
\leftrightarrow
(  [\pi  ] \varphi \wedge  [\pi' ] \varphi )
 \tagLabel{Red$_{\cup}$}{ax:Prg2}\\
      &  ([\pi  ] \varphi
     \wedge [\pi'  ] \psi )
\rightarrow [\pi \cap \pi'   ] (\varphi \wedge \psi)
 \tagLabel{Add1$_{\cap}$}{ax:Prg4}\\
      &  (\langle \pi  \rangle x
     \wedge \langle  \pi' \rangle x)
\rightarrow  \langle \pi \cap \pi'  \rangle x
 \tagLabel{Add2$_{\cap}$}{ax:Prg3}\\
       &  \varphi
\rightarrow [\pi   ]
\langle -\pi  \rangle \varphi
 \tagLabel{Conv1$_{-}$}{ax:Prg5}\\
     &  \varphi
\rightarrow [-\pi   ]
\langle \pi  \rangle \varphi
 \tagLabel{Conv2$_{-}$}{ax:Prg6}\\
       &  ([\preceq_{i,\tau}  ]\varphi
       \wedge
       [\preceq_{i,\tau}^\sim  ]\varphi)
\leftrightarrow [\equiv_i   ]\varphi
 \tagLabel{Comp1$_{\sim}$}{ax:Prg7}\\
      &  \langle \preceq_{i,\tau} \rangle x
\rightarrow [\preceq_{i,\tau}^\sim   ]\neg x
 \tagLabel{Comp2$_{\sim}$}{ax:Prg8}\\
       &  [?\varphi] \psi
\rightarrow (\varphi \rightarrow \psi)
 \tagLabel{Red$_{?}$}{ax:Prg9}\\
       &  \langle \pi  \rangle (x \wedge \varphi)
       \rightarrow [\pi'](x \rightarrow \varphi)
 \tagLabel{Most$_{x}$}{ax:Most}\\
        &  \frac{\varphi }{[\pi] \varphi }
 \tagLabel{Nec$_{\pi}$}{ax:Nec}\\
         &  \frac{   [\pi ]   \neg x  \text{ for  all } x \in \mathit{Nom}  }{[\pi ] \bot }
 \tagLabel{Cov}{ax:Name}
\end{align}

\end{definition}

Note that the primitive operators
$[\preceq_{i,P}]$
and $[\preceq_{i,D}]$
are S4 (or KT4), while
$[\equiv_i]$
is S5.
The only interaction
principles
between these three operators
are the 
``inclusion'' Axiom 
\ref{ax:Int}
and the
``connectedness'' Axiom \ref{ax:Conn}.
Operators
$[\preceq_{i,P}]$
and $[\preceq_{i,D}]$
do not interact since, as we have emphasized
in Section \ref{LogicPart},
epistemic plausibility and
desirability
are assumed to be independent notions.

For every $\varphi \in \mathcal{L}_{\mathrm{DLCA} }$,
we write $\vdash \varphi$
to denote the fact that $\varphi $
is a theorem of $\mathrm{DLCA}$, i.e.,
there exists an at most countably infinite sequence 
$\psi_0, \psi_1, \ldots$
such that
$\psi_0= \varphi$
 and for all $k \geq 0$,
 $\psi_k$ is an instance of some
 axiom or  $\psi_k$
 can be obtained from some later members of the sequence by an application of some inference rule.

The rest of this section is devoted
to prove that 
the logic
$\mathrm{DLCA}$
is sound and complete for the class of
multi-agent cognitive models.

Soundness, namely checking that the
axioms
are valid and the the rules of inferences
preserve validity, is a routine exercise.
Notice that the admissibility
of the rule of inference
\ref{ax:Name}
is guaranteed by the fact that
the set of nominals $\mathit{Nom}$
is infinite.

As for completeness, the proof
is organized in several steps.
We use techniques from
dynamic logic and
modal logic with names \cite{PassyTinchev,Goranko93}.



In the rest of this section,
we denote sets of formulas 
from $\mathcal{L}_{\mathrm{DLCA} }$
by $\Sigma, \Sigma', \ldots$.
Let $\varphi \in \mathcal{L}_{\mathrm{DLCA} }$
and $\Sigma \subseteq \mathcal{L}_{\mathrm{DLCA} }$, we define:
\begin{align*}
\Sigma + \varphi= \{ \psi  \in  \mathcal{L}_{\mathrm{DLCA} }:  \varphi \rightarrow \psi \in \Sigma\}.
\end{align*}

Let us start by defining the concepts
of theory and maximal consistent theory.
\begin{definition}[Theory]
A set of formulas $\Sigma$ is said to be a theory if it
contains
all theorems of $\mathrm{DLCA}$
and is closed under modus ponens
and rule \ref{ax:Name}. It is said to be a consistent theory if
it is a theory and
 $\bot \not \in \Sigma$.
It is said to be a maximal consistent theory (MCT) if it
is a consistent theory
and, for each consistent theory $\Sigma'$,
we have  that if
$\Sigma\subseteq \Sigma'$ then $\Sigma= \Sigma'$.
\end{definition}

We have the following property for theories.

\begin{proposition}\label{proptheories}
Let
$\Sigma$ be a theory
and let $\varphi \in \mathcal{L}_{\mathrm{DLCA} }$.
Then, $\Sigma + \varphi$ is a theory.
Moreover,
if $\Sigma$
is consistent then
 either $\Sigma +  \varphi$ 
is consistent or  $\Sigma + \neg  \varphi$  is consistent.

\end{proposition}
\begin{proof}
Let us first prove that  if $\Sigma $
is a  theory then $\Sigma + \varphi$
is a theory as well. 
Suppose $\Sigma $
is a  theory.
Then,
$\Sigma + \varphi$ clearly contains 
all theorems of $\mathrm{DLCA}$.
Moreover, 
suppose $\psi \rightarrow \psi', \psi \in \Sigma + \varphi$.
Thus, by definition of $\Sigma + \varphi$,
we have
$ \varphi \rightarrow \psi, \varphi \rightarrow (\psi \rightarrow \psi') \in  \Sigma$.
Since  $\Sigma$
is closed under modus ponens
and contains all theorems of $\mathrm{DLCA}$,
the latter implies $ (\varphi \rightarrow \psi )\wedge \big( \varphi \rightarrow (\psi \rightarrow \psi')\big) \in  \Sigma$.
Consequently, 
since $\Sigma$
is closed under modus ponens,
$  \varphi \rightarrow  \psi' \in  \Sigma$.
Hence, $\psi' \in \Sigma + \varphi $.
This means that $\Sigma + \varphi $
is closed under modus ponens.
Finally, let us show that $\Sigma + \varphi$
is closed under \ref{ax:Name}.
Suppose $[\pi] \neg x \in \Sigma + \varphi$
for all $x$.
Thus,
by definition of $\Sigma + \varphi$,
$\varphi \rightarrow [\pi] \neg x \in \Sigma $
for all $x$.
Since $\Sigma$ is a theory,
the latter implies that 
$[?\varphi {;} \pi] \neg x \in \Sigma $
for all $x$. Thus, since $\Sigma$ is a theory,
$[?\varphi {;} \pi] \bot \in \Sigma $
and, consequently, $\varphi \rightarrow [\pi] \bot \in \Sigma $.
It follows that $[\pi] \bot \in \Sigma + \varphi$.

Let us show that 
if $\Sigma$
is consistent then
either $\Sigma +  \varphi$ 
is consistent or  $\Sigma + \neg  \varphi$  is consistent.
Suppose the antecedent is true
while the consequent is false. Then, $\varphi \rightarrow \bot \in \Sigma$
and  $\neg \varphi \rightarrow \bot \in \Sigma$.
Since $\Sigma $
is a theory, we have $(\varphi \rightarrow \bot ) \wedge (\neg \varphi \rightarrow \bot ) \in \Sigma$.
Thus, $\bot \in \Sigma$ which is in contradiction with the fact that $\Sigma
$ is consistent.
\end{proof}

The following proposition highlights some standard properties
of MCTs.

\begin{proposition}\label{beflinde}
Let
$\Sigma$ be a MCT.
Then,
for all $\varphi, \psi \in \mathcal{L}_{\mathrm{DLCA} }$:
\begin{itemize}

\item $\varphi \in \Sigma$
or $\neg \varphi \in \Sigma$,

\item   $\varphi \vee \psi \in \Sigma$
iff $\varphi  \in \Sigma$
or $ \psi \in \Sigma$.



\end{itemize}

\end{proposition}
\begin{proof}
We only prove the first item
by reductio ad absurdum.
Suppose $\Sigma$ is a MCT, 
$\varphi \not \in \Sigma$
and $\neg \varphi \not \in \Sigma$.
We clearly have $\Sigma \subseteq \Sigma + \varphi$
and  $\Sigma \subseteq \Sigma + \neg \varphi$.
Moreover, $\varphi \in \Sigma + \varphi$
and $\neg \varphi \in \Sigma + \neg \varphi$. 
Thus, $\Sigma \subset \Sigma + \varphi$
and  $\Sigma \subset \Sigma + \neg \varphi$.
By Proposition
\ref{proptheories},
$\Sigma + \varphi$
and 
$\Sigma + \neg \varphi$
are theories.
Moreover,
either
$ \Sigma + \varphi$
is consistent 
or $\Sigma + \neg \varphi$
is consistent. This contradicts the fact that 
$\Sigma$
is a MCT.
\end{proof}

The following variant of the Lindenbaum's lemma
is proved in the same way
as  \cite[Lemma 4.15]{PassyTinchev}.

\begin{lemma}\label{linden}
Let  $\Sigma$
be a consistent theory
and let $\varphi \not \in \Sigma$.
Then, there exists a MCT $\Sigma^+$
such that $\Sigma \subseteq \Sigma^+$
and $\varphi \not \in \Sigma^+$.
\end{lemma}

The following lemma highlights
a fundamental property of MCTs.
\begin{lemma}\label{containsone}
Let $\Sigma$ be a MCT. Then,
there exists $x \in \mathit{Nom}$
such $x \in \Sigma$.
\end{lemma}
\begin{proof}
We prove the lemma by reductio ad absurdum.
Let $\Sigma$
be a MCT.
Moreover,
suppose that,
for all $x \in \mathit{Nom}$,
$x \not \in \Sigma$.
By Proposition \ref{beflinde},
it follows that, for all $x \in \mathit{Nom}$,
$\neg x \in \Sigma$.

By Axiom \ref{ax:Prg9},
we have $   \neg x \leftrightarrow   [?\top] \neg x     \in \Sigma $
 for all $x \in \mathit{Nom}$.
Thus,  for all $x \in \mathit{Nom}$,
$ [?\top] \neg x     \in \Sigma $.
Hence, since $\Sigma$ is closed under  \ref{ax:Name},
 $[?\top] \bot     \in \Sigma $.
 By Axiom \ref{ax:Prg9},
 the latter is equivalent to
  $\bot     \in \Sigma $.
  The latter is contradiction
  with the fact that $\Sigma$
  is a MCT.
\end{proof}

%
%

Let us now define the canonical model for our logic.

\begin{definition}[Canonical model]\label{defCanModel}
The canonical
model
is the tuple
$M^c = (  W^c,( \preceq_{i,P}^c)_{i \in \mathit{Agt} }, (\preceq_{i,D}^c)_{i \in \mathit{Agt} }  ,
(\equiv_i^c)_{i \in \mathit{Agt} }  , V^c  )$
such that:
\begin{itemize}
\item  $W^c$ is the set of all MCTs,

\item
for all $i \in \mathit{Agt}$,
for all $\tau \in \{P,D\}$,
for all $w,v \in W^c$,
$w \preceq_{i,\tau}^c v$ iff,
for all $\varphi \in \mathcal{L}_{\mathrm{DLCA} } $,
if $[\preceq_{i,\tau}]\varphi \in w$
then $\varphi \in v$,

\item
for all $i \in \mathit{Agt}$,
for all $w,v \in W^c$,
$w \equiv_i^c v$ iff,
for all $\varphi \in \mathcal{L}_{\mathrm{DLCA} } $,
if $[\equiv_i]\varphi \in w$
then $\varphi \in v$,

\item
for all $w\in W^c$,
$ V^c(w) = (\mathit{Atm} \cup \mathit{Nom}) \cap w $.

\end{itemize}
 \end{definition}

 Let us now define the canonical
 relations for the complex programs $\pi$.

 \begin{definition}[Canonical relation]\label{canonicrel}
 Let
$M^c= (  W^c,( \preceq_{i,P}^c)_{i \in \mathit{Agt} }, (\preceq_{i,D}^c)_{i \in \mathit{Agt} }  ,
(\equiv_i^c)_{i \in \mathit{Agt} }  , V^c  )$
 be the canonical model.
 Then, for all
$\pi \in \mathcal{P}$
and for all $w,v \in W^c$:
\begin{align*}
w R_\pi^c v \text{ iff, for all }
\varphi \in \mathcal{L}_{\mathrm{DLCA} }, \text{ if }
[\pi]\varphi \in w
\text{ then }
\varphi \in v.
\end{align*}

 \end{definition}

 The following Lemma \ref{propcanonical} highlights
 one fundamental property
 of the canonical model.

\begin{lemma}\label{propcanonical}
Let $M^c = (  W^c,( \preceq_{i,P}^c)_{i \in \mathit{Agt} }, (\preceq_{i,D}^c)_{i \in \mathit{Agt} }  ,
(\equiv_i^c)_{i \in \mathit{Agt} }  , V^c  )$
be the canonical model.
Then, for all $\Sigma, \Sigma' \in W^c$,
for all
$\pi \in \mathcal{P}$
and  for all $x \in \mathit{Nom}$,
if $x \in \Sigma,  x \in \Sigma' $
and $\Sigma R_\pi^c \Sigma'$
then $\Sigma = \Sigma'$.
\end{lemma}
\begin{proof}
 Let us first prove that
(i)
if
$x \in \Sigma$
and
$\varphi \in \Sigma$
then
$[\pi] (x \rightarrow \varphi) \in \Sigma$.
Suppose
$x, \varphi \in \Sigma$.
Thus, $x \wedge \varphi \in \Sigma$
since $\Sigma$
is a MCT.
Moreover, 
$ (x \wedge \varphi)
       \rightarrow [\pi](x \rightarrow \varphi) \in \Sigma$,
because of Axiom \ref{ax:Most}.
Hence,
$[\pi](x \rightarrow \varphi) \in \Sigma$.

Now let us prove
by absurdum
that (ii)
if $x \in \Sigma, \Sigma' $
and $  \Sigma R_\pi^c \Sigma' $
then $\Sigma = \Sigma'$.
Suppose
$x \in \Sigma, \Sigma' $,
 $ \Sigma R_\pi^c \Sigma' $
and $\Sigma \neq \Sigma'$.
The latter implies
that there exists $\varphi$
such that $\varphi \in \Sigma$
and $ \varphi \not \in \Sigma'$.
By item (i) above, it follows that
$[\pi](x \rightarrow \varphi) \in \Sigma$.
Since $\Sigma R_\pi^c \Sigma' $,
the latter implies
that
$x \rightarrow \varphi \in \Sigma'$.
Since
$x \in \Sigma' $,
it follows that $\varphi \in \Sigma'$
which leads to a contradiction.
\end{proof}

The next step consists in proving the following existence lemma.

\begin{lemma}\label{existencelemma}

Let $M^c = (  W^c,( \preceq_{i,P}^c)_{i \in \mathit{Agt} }, (\preceq_{i,D}^c)_{i \in \mathit{Agt} }  ,
(\equiv_i^c)_{i \in \mathit{Agt} }  , V^c  )$
be the canonical model,
let $w \in W^c$
and let $\langle \pi \rangle \varphi \in
\mathcal{L}_{\mathrm{DLCA} }$.
Then,
if
$\langle \pi \rangle \varphi \in w$
then there exists $v \in W^c$
such that
$w R_\pi^c v$
and $\varphi \in v$.

\end{lemma}
\begin{proof}
 Suppose 
$w$ is a MCT and  $\langle \pi \rangle \varphi \in w$.
It follows  that $[\pi]w = \{\psi : [\pi]\psi \in w \}$
is a consistent theory.
Indeed,
it is easy to check that 
$[\pi]w$ contains 
all theorems of $\mathrm{DLCA}$,
is closed under modus ponens
and rule \ref{ax:Name}. Let us prove that it is consistent
by reductio ad absurdum.
Suppose $\bot \in [\pi]w$.
Thus, $ [\pi]\bot \in w$.
Hence, $ [\pi]\neg\varphi \in w$.
Since $\langle \pi \rangle \varphi \in w$,
$\bot \in w$.
The latter
contradicts the fact that $w$ is a MCT.
Let us distinguish two cases.

Case 1: $\varphi \in [\pi]w$. 
Thus, $\neg \varphi \not \in [\pi]w$
since $w$ is consistent. Thus,
by Lemma \ref{linden}, there exists MCT
$v$
such that $ [\pi]w \subseteq v $,
$\varphi \in v$ and
$\neg \varphi \not \in v$. By definition of $R_\pi^c$,
$w R_\pi^c v$.

Case 2: $\varphi \not \in [\pi]w$. 
By Proposition \ref{proptheories}, $ [\pi]w + \varphi$
is a theory since $ [\pi]w$
is a theory. $ [\pi]w + \varphi$
is consistent. Suppose it is not.
Thus, $\varphi \rightarrow \bot \in  [\pi]w$
and, consequently, $\neg \varphi\in  [\pi]w $.
Hence, $[\pi] \neg \varphi \in w$.
It follows that $\bot \in w$,
since $\langle \pi \rangle \varphi \in w$.
But this contradicts the fact that $w$ is a MCT.
Thus, $ [\pi]w + \varphi$ is a consistent theory.
Moreover, $\varphi \in [\pi]w + \varphi $,
$\neg \varphi \not \in [\pi]w + \varphi $
and $ [\pi]w \subseteq  [\pi]w + \varphi $.
By Lemma \ref{linden}, there exists MCT
$v$
such that $ [\pi]w \subseteq v $,
$\varphi \in v$ and
$\neg \varphi \not \in v$. By definition of $R_\pi^c$,
$w R_\pi^c v$.
\end{proof}

The following truth lemma is proved in the usual way
by
induction on the structure of $\varphi$
thanks to Lemma \ref{existencelemma}.

\begin{lemma}\label{truthlemma}
Let $M^c = (  W^c,( \preceq_{i,P}^c)_{i \in \mathit{Agt} }, (\preceq_{i,D}^c)_{i \in \mathit{Agt} }  ,
(\equiv_i^c)_{i \in \mathit{Agt} }  , V^c  )$
be the canonical model,
let $w \in W^c$
and let $  \varphi \in
\mathcal{L}_{\mathrm{DLCA} }$.
Then,
$M^c , w \models \varphi$
iff $\varphi \in w$.

\end{lemma}
\begin{proof}
The proof is by induction
on the structure of $\varphi$.
We only prove the case in which $\varphi$
is a modal formula $[\pi]\psi$.
As for the right-to-left direction
we have:
\begin{eqnarray*}
[\pi]\psi \in w & \text{ only if } & \forall v \in R_\pi^c(w): \psi \in v \text{ (by definition of $R_\pi^c$)} \\
  & \text{ iff } & \forall v \in R_\pi^c(w): M^c , v \models \psi \text{ (by induction hypothesis)} \\
  & \text{ iff } & M^c , w \models [\pi]\psi
\end{eqnarray*}
As for the left-to-right direction,
we prove that if $[\pi]\psi \not \in w $
 then $M^c , w \not \models [\pi]\psi $
 that, given the property of MCSs, is equivalent to proving that
 if
 $\langle \pi \rangle \psi \in w $
 then $M^c , w  \models \langle \pi \rangle \psi$.
 Suppose $\langle \pi \rangle \psi \in w $.
 Then, by Lemma \ref{existencelemma},
 there exists    $v \in W^c$
such that
$w R_\pi^c v$
and $\psi \in v$.
Hence,
by induction hypothesis,
 there exists    $v \in W^c$
such that
$w R_\pi^c v$
and $M^c ,v\models  \psi $.
The latter is equivalent to
$M^c , w  \models \langle \pi \rangle \psi$.
\end{proof}

The pre-final stage of
the proof
consists in introducing 
an alternative semantics 
for the language $  
\mathcal{L}_{\mathrm{DLCA} }$
which turns out to be equivalent
to the original semantics based on MCMs.

\begin{definition}[Quasi multi-agent cognitive model]\label{sec:quasiMCMdef}
A quasi multi-agent cognitive model (quasi-MCM) is a tuple
$M = (  W,( \preceq_{i,P})_{i \in \mathit{Agt} }, (\preceq_{i,D})_{i \in \mathit{Agt} }  ,
(\equiv_i)_{i \in \mathit{Agt} }  , V  )$
where $W$, $\preceq_{i,P}$, $\preceq_{i,D}$, $\equiv_i$ and $V$
are as in Definition \ref{DefSemanticsModel}
except that Constraint C4 is replaced by the following weaker constraint.
For all $w,v \in W$:
\begin{description}

\item[(C4$^*$)] if $ V_{\mathit{Nom}}(w) \cap V_{\mathit{Nom}}(v)  \neq \emptyset $
and $w R_\pi v$ for some $\pi \in  \mathcal{P}$
 then $w =v$.

\end{description}

\end{definition}

By the generated submodel property,
it is easy
to show that the semantics in terms
of MCMs and the semantics
in terms of quasi-MCMs
are equivalent with respect to the language 
$  
\mathcal{L}_{\mathrm{DLCA} }$.

\begin{proposition}\label{equivsemantics}
Let $\varphi \in \mathcal{L}_{\mathrm{DLCA} }$.
Then, $\varphi$
is valid relative to the class of MCMs
if and only if
$\varphi$
is valid relative to the class of quasi-MCMs.
\end{proposition}

The following theorem highlights
that the canonical
model
is indeed a structure
of the right type.

\begin{lemma}\label{canonicity}
The canonical
model $M^c$
is a quasi-MCM.
\end{lemma}
\begin{proof}
The fact that
$M^c$ satisfies Constraints C3 and C4$^*$ 
follows from Lemma \ref{containsone} and Lemma \ref{propcanonical}.
To prove that
$\equiv_i$
is an equivalence relation
 that $\preceq_{i,D}^c$
 and $\preceq_{i,D}^c$
are  preorders and
that $M^c$ satisfies Constraints C1 and C2
is just a routine exercise. Indeed,
Axioms \ref{ax:Tequiv},
\ref{ax:4equiv},
\ref{ax:5equiv},
\ref{ax:Tpreceq},
\ref{ax:4preceq}
\ref{ax:Int}
and \ref{ax:Conn}
are canonical for these semantic conditions.

To conclude, we need to prove that the following six conditions hold,
for $i \in \mathit{Agt}$ and $\tau \in \{P,D\}$:
  \begin{align*}
   (w,v) \in  R_{ \preceq_{i,\tau}^\sim}^c    & \text{ iff }
    (w,v) \in  R_{ \equiv_i }^c \text{ and }
 (w,v) \not \in  R_{ \preceq_{i,\tau} }^c \\
    (w,v) \in R_{ \pi ;  \pi'}^c    & \text{ iff }
    \exists u \in W^c:
 (w,u) \in R_{ \pi }^c
 \text{ and }  (u,v) \in  R_{  \pi'}^c  \\
    (w,v) \in R_{ \pi \cup  \pi'}^c    & \text{ iff }
 (w,v) \in R_{ \pi }^c
 \text{ or }  (w,v) \in  R_{  \pi'}^c  \\
 (w,v) \in R_{ \pi \cap  \pi'}^c    & \text{ iff }
 (w,v) \in R_{ \pi }^c
 \text{ and }  (w,v) \in R_{  \pi'}^c  \\
  (w,v) \in  R_{ -\pi}^c    & \text{ iff }
 (v,w) \in  R_{ \pi}^c \\
 w R_{\varphi?}^c v   & \text{ iff }
w =v  \text{ and } M^c,w \models \varphi
\end{align*}
We only prove the second and fourth  conditions which
are the most difficult ones to prove.

Let us start with the proof of the second condition.
The right-to-left direction is standard.
We only prove the left-to-right direction. 
Suppose $   (w,v) \in R_{ \pi ;  \pi'}^c$.
Let  $[\pi]w = \{\psi : [\pi]\psi \in w \}$.
Moreover, let $  \langle \pi' \rangle v  = \{  \langle \pi' \rangle \psi : \psi \in v \} $.
Finally, let $\langle \pi' \rangle \psi_1, \langle \pi' \rangle \psi_2, \ldots$ be an enumeration of the elements of $\langle \pi' \rangle v$.
We define $\Sigma^1= [\pi]w +   \langle \pi' \rangle \psi_1$
and, for all $k >1$,  $\Sigma^k= \Sigma^{k-1} +   \langle \pi' \rangle \psi_k$.
By Lemma \ref{proptheories} and the fact that $[\pi]w$ is a theory,
it can be shown that 
every  $\Sigma^k$ is a theory. 
Moreover, by induction on $k$,
it can be shown that every 
$\Sigma^k$ is consistent. 
Since  $\Sigma^{k-1} \subseteq \Sigma^{k}$
for all $k >1$, it follows that $ \Sigma  = \bigcup_{k > 1} \Sigma^{k-1}$
is a consistent theory.  By Lemma \ref{linden} and the definition of $\Sigma$, there exists $u \in W^c$
such that $\Sigma \subseteq u $, $   (w,u) \in R_{ \pi }^c$ and $   (u,v) \in R_{   \pi'}^c$.

Let us now prove the fourth condition.
Suppose
$ (w,v) \in R_{ \pi \cap  \pi'}^c $.
By Definition \ref{canonicrel}
and Proposition \ref{beflinde},
it follows that,
for all $\varphi$,
if $\varphi \in v  $
then $\langle \pi \cap  \pi' \rangle \varphi \in w$.
The latter implies that
for all $\varphi$,
if $\varphi \in v  $
then $\langle \pi \cap  \pi' \rangle (\varphi \vee
\bot )
\in w$
since
$\vdash \langle \pi \cap  \pi' \rangle \varphi
\rightarrow
\langle \pi \cap  \pi' \rangle (\varphi \vee
\bot )$.
By Axiom \ref{ax:Kequiv},
it follows that,
for all $\varphi$,
if $\varphi \in v  $
then $\langle \pi  \rangle \varphi
\vee
\langle \pi'  \rangle \bot
\in w$.
Thus,
for all $\varphi$,
if $\varphi \in v  $
then $\langle \pi  \rangle \varphi
\in w$,
since
$\vdash
(\langle \pi  \rangle \varphi
\vee
\langle \pi'  \rangle \bot)
\rightarrow
\langle \pi  \rangle \varphi $.
In a similar way,
we can prove
that,
for all $\varphi$,
if $\varphi \in v  $
then $\langle \pi'  \rangle \varphi
\in w$.
By
Definition \ref{canonicrel}
and Proposition \ref{beflinde},
it follows that
$ (w,v) \in R_{ \pi }^c
 \text{ and }  (w,v) \in R_{  \pi'}^c$.

 Now suppose
 $ (w,v) \in R_{ \pi }^c
 \text{ and }  (w,v) \in R_{  \pi'}^c$.
Thus,
by Definition \ref{canonicrel}
and Proposition \ref{beflinde}, (i) for all $\varphi$,
if $\varphi \in v  $
then $\langle \pi  \rangle \varphi \in w$
and $\langle   \pi' \rangle \varphi \in w$.
By Proposition \ref{beflinde} and Lemma \ref{containsone}, we have that
(ii) there exists $x \in \mathit{Nom}$
such that,
for all $\varphi$,
$\varphi \in v  $ iff
$x \wedge \varphi  \in v  $.
Item (i) and item (ii)
together imply that
(iii)
there exists $x \in \mathit{Nom}$
such that,
for all $\varphi$,
if $\varphi \in v  $
then $\langle \pi  \rangle (x \wedge \varphi ) \in w$
and $\langle   \pi' \rangle (x \wedge \varphi ) \in w$.
We are going to prove
the following
theorem:
\begin{align*}
\vdash (\langle \pi  \rangle (x \wedge \varphi )
\wedge
\langle   \pi' \rangle (x \wedge \varphi ))
\rightarrow
\langle \pi \cap  \pi' \rangle (x \wedge \varphi )
\end{align*}
By Axiom \ref{ax:Kequiv},
$\langle \pi  \rangle (x \wedge \varphi )
\wedge
\langle   \pi' \rangle (x \wedge \varphi )$
implies
$\langle \pi  \rangle x
\wedge
\langle   \pi' \rangle x $.
By Axiom \ref{ax:Prg3}, the latter implies
$\langle \pi \cap \pi'  \rangle x$.
Moreover,
by Axiom \ref{ax:Int} and Axiom \ref{ax:Most},
$\langle \pi  \rangle (x \wedge \varphi )$
implies
$[ \equiv_\emptyset  ] (x \rightarrow \varphi )$.
By Axiom \ref{ax:Int},
the latter implies
$[ \pi \cap \pi'  ] (x \rightarrow \varphi )$.
By Axiom \ref{ax:Kequiv},
$[ \pi \cap \pi'  ] (x \rightarrow \varphi )$
and
$\langle \pi \cap \pi'  \rangle x$
together imply
$\langle \pi \cap \pi'  \rangle ( x \wedge \varphi)$.
Thus,
$\langle \pi  \rangle (x \wedge \varphi )
\wedge
\langle   \pi' \rangle (x \wedge \varphi )$
implies $\langle \pi \cap \pi'  \rangle ( x \wedge \varphi)$.

From previous item (iii) and the previous theorem
it follows that
there exists $x \in \mathit{Nom}$
such that,
for all $\varphi$,
if $\varphi \in v  $
then $\langle \pi \cap \pi'  \rangle ( x \wedge \varphi)$.
The latter implies that,
for all $\varphi$,
if $\varphi \in v  $
then $\langle \pi \cap \pi'  \rangle \varphi$.
The latter implies that
 $ (w,v) \in R_{ \pi \cap \pi' }^c$.
\end{proof}

Let us conclude the proof by supposing $\not \vdash  \neg \varphi $.
Therefore, by Lemma \ref{linden} and the fact that 
the set of $\mathrm{DLCA}$-theorems
is a consistent theory, there exists a MCT $w$ such that $\neg \varphi \not \in w$.
Thus, 
by  Proposition \ref{beflinde}, 
we can find a MCT $w$ such that $\varphi  \in w$.
By Lemma \ref{truthlemma}, the latter implies
$M^c, w \models \varphi$
for some $w \in W^c$. 
Since, by Lemma \ref{canonicity}, $M^c$
is a quasi-MCM,
it follows that $\varphi$
is satisfiable relative to the class of quasi-MCMs.
Therefore,
by Proposition \ref{equivsemantics}, $\varphi$
is satisfiable relative to the class
of MCMs.

We can finally state the main result of this section.
 \begin{theorem}\label{firstcomplete}
The logic
$\mathrm{DLCA}$
is sound and complete for the class of
multi-agent cognitive models.
\end{theorem}

\section{Application to Game Theory}\label{sec:GTapplication}

In this section, 
we apply our logical framework
to the analysis of 
single-stage games under  incomplete information in which 
 agents
 only play once (i.e., interaction is non-repeated)
and
may not know some
relevant characteristic of others
including their preferences, choices and beliefs.

Let $ \mathit{Act}$
be a set of action names with elements noted $a, b , \ldots$
Let
a joint action 
be a function $\delta : \mathit{Agt} \longrightarrow \mathit{Act}$
and
the set of joint actions
be denoted by $ \mathit{JAct}$.

For every coalition $C \in  2^\mathit{Agt} $
and for every $\delta \in  \mathit{JAct} $, let
$\delta_C$
be the $C$-restriction of $\delta$,
that is,
the function 
$\delta_C : C \longrightarrow \mathit{Act}$
such that $\delta_C(i)=\delta(i)$
for all $i \in C$.
For notational convenience,
we write $-i$
instead of $\mathit{Agt} \setminus\{i\}$,
with $i \in \mathit{Agt} $.

In order to model strategic interaction
in our setting, we extend MCMs of Definition \ref{DefSemanticsModel} by agents' choices.
We call MCM with choices the resulting models.
\begin{definition}[Multi-agent cognitive model with choices]\label{DefSemanticsModel2}
A multi-agent cognitive model  with choices (MCMC) is a tuple
$M = (  W,( \preceq_{i,P})_{i \in \mathit{Agt} }, (\preceq_{i,D})_{i \in \mathit{Agt} }  ,\\
(\equiv_i)_{i \in \mathit{Agt} }  , ( \mathit{C}_i) _{i \in \mathit{Agt} }, V  )$,
where
$M = (  W,( \preceq_{i,P})_{i \in \mathit{Agt} }, (\preceq_{i,D})_{i \in \mathit{Agt} }  ,
(\equiv_i)_{i \in \mathit{Agt} }  ,  V  )$
is a MCM and every $C_i$
is a choice function  $\mathit{C}_i : W  \longrightarrow  \mathit{Act}$,
which satisfies the following constraint, for each $i \in \mathit{Agt}$ and $\delta \in  \mathit{JAct}$:
\begin{description}
\item[(C5)] if $\forall j \in \mathit{Agt}$, 
$\exists w_j \in W$
such that
$w \equiv_i w_j$  and
$\mathit{C}_j (w_j) = \delta(j) $, then
$\exists v \in W$ such that $w \equiv_i v $
and, $\forall j \in \mathit{Agt}$, $\mathit{C}_{j} (v) = \delta(j)$.

\end{description}

\end{definition}

For every $w \in W$,
$\mathit{C}_i(w)$
denotes agent $i$'s actual choice at $w$. 
If $w \equiv_i v$
and $\mathit{C}_j(v)=a$, then $a$
is a potential choice of agent $j$
from agent $i$'s perspective. 

According to Constraint C5, agents' choices are subjectively independent, in the sense
that
 every agent $i$
 knows that an agent cannot be deprived
 of her choices due to the choices made by the others.
In other words, suppose that, from
 agent $i$'s perspective,
 $ \delta(j)$
 is a potential choice of $j$
 for every agent $j$. 
  Then,
 from agent $i$'s perspective, there should be a state
 at which the agents choose the joint action $\delta$.
It is a subjective version of the property of choice independence
formulated in the  ``seeing to it that'' ($\mathrm{STIT}$)
framework
 \cite{belnap01facing,LoriniJANCL,BalbianiHerzigTroquard2008}.

At the syntactic level,
we extend the language $\mathcal{L}_{\mathrm{DLCA}}$ by
special
constants for choices of type $ \mathsf{play}(i,a)$,
with $i \in \mathit{Agt}$
and $a \in \mathit{Act}$,
denoting the fact that 
 ``agent $i$ plays (or chooses) action $a$''.
 The resulting language is noted $\mathcal{L}_{\mathrm{DLCAG}}$,
 where $\mathrm{DLCAG}$
 stands for  ``Dynamic Logic of Cognitive Attitudes in Games''
 and a constant 
$ \mathsf{play}(i,a)$ is interpreted relative to a MCMC $M$
and a world $w$ in $M$, as follows:
\begin{eqnarray*}
M, w \models  \mathsf{play}(i,a) & \Longleftrightarrow & \mathit{C}_i (w) = a .
\end{eqnarray*}

Let $\delta \in \mathit{JAct}$
and  $C \in  2^\mathit{Agt} $.
We define:
\begin{align*}
\mathsf{play}(\delta_{C})& =_{\mathit{def}} \bigwedge_{i \in C} \mathsf{play}\big(i, \delta_{C}(i) \big).
\end{align*}

For every formula $\varphi$
in $\mathcal{L}_{\mathrm{DLCAG} }$
we say
that $\varphi$
is valid, noted $\models_{\mathit{MCMC}} \varphi$,
if and only if
for every multi-agent cognitive model
with choices
$M$
and world $w$
in $M$,
we have $M,w \models \varphi$.


\begin{definition}[Logic $\mathrm{DLCAG}$]\label{axiomatics2}
We define  $\mathrm{DLCAG}$
to be the extension of logic 
$\mathrm{DLCA}$
given by the following
 axioms:
\begin{align}
& \mathsf{play}(i,a) \rightarrow \neg \mathsf{play}(i,b) \text{ if } a \neq b
 \tagLabel{MostAct}{ax:most}\\
& \bigvee_{a \in  \mathit{Act}}\mathsf{play}(i,a) 
 \tagLabel{LeastAct}{ax:least}\\
 & \Big( \bigwedge_{j \in \mathit{Agt}} \langle \equiv_i \rangle  \mathsf{play}\big(j,\delta(j) \big) \Big)
 \rightarrow \langle \equiv_i \rangle  \mathsf{play}\big(\delta_{\mathit{Agt}} \big)
 \tagLabel{SIC}{ax:indep}
\end{align}
\end{definition}
Axiom \ref{ax:most}
means that an agent  chooses at most one action from $\mathit{Act}$ while,
according to Axiom \ref{ax:least}, 
an agent  chooses at least one action
from $\mathit{Act}$.
Axiom \ref{ax:indep}
is the syntactic counterpart of 
\emph{subjective}
choice independence 
expressed by Constraint C5.


We can adapt
the techniques 
used for proving 
Theorem
\ref{firstcomplete}
in order
to prove the following
Theorem
\ref{firstcomplete2}.
 \begin{theorem}\label{firstcomplete2}
The logic
$\mathrm{DLCAG}$
is sound and complete for the class of
multi-agent cognitive models
with choices.
\end{theorem}

\begin{proof}
Verifying that the logic 
$\mathrm{DLCAG}$
is sound for the class of
MCMCs is a routine exercise.
As for completeness,
the proof is just a straightforward adaptation of the proof
of completeness of the logic $\mathrm{DLCA}$.
First,
we need to define corresponding notions of theory
and maximal consistent theory for the logic
$\mathrm{DLCAG}$
which are akin to the ones for the logic
$\mathrm{DLCA}$
and use them
to define the canonical model
and the canonical relation 
for $\mathrm{DLCAG}$.
The canonical model
for $\mathrm{DLCAG}$
is defined to be a tuple 
$M^c = (  W^c,( \preceq_{i,P}^c)_{i \in \mathit{Agt} }, (\preceq_{i,D}^c)_{i \in \mathit{Agt} }  ,
(\equiv_i^c)_{i \in \mathit{Agt} }  , ( \mathit{C}_i^c) _{i \in \mathit{Agt} } ,V^c  )$
where 
$W^c$ is the set of all MCTs
for $\mathrm{DLCAG}$, 
$\preceq_{i,P}^c$,
$\preceq_{i,D}^c$,
$\equiv_i^c$
and $V^c$
are defined as in the
definition
of the canonical
model for  
 $\mathrm{DLCA}$ (Definition \ref{defCanModel}),
 and
 $\mathit{C}_i^c : W^c \longrightarrow 2^{ \mathit{Act}}$
 such that, for all $a \in \mathit{Act}$ and $w \in W^c$,
 $a \in \mathit{C}_i^c(w)$
 if and only if $ \mathsf{play}(i,a) \in w$.
 The canonical
 relation 
 for $\mathrm{DLCAG}$
 is defined in the same way
 as the canonical relation for $\mathrm{DLCA}$
 (Definition \ref{canonicrel}).

It is immediate to adapt
the proof of the
existence and truth lemma
for $\mathrm{DLCA}$
(Lemma \ref{existencelemma} and Lemma \ref{truthlemma})
to prove corresponding existence and truth lemma
for  $\mathrm{DLCAG}$.

Secondly, we need to define the notion
of quasi multi-agent cognitive model with choices (quasi-MCMC) 
which is analogous to the definition of quasi-MCM (Definition \ref{sec:quasiMCMdef}).
In particular, a quasi-MCMC is defined to be a tuple 
$M = (  W,( \preceq_{i,P})_{i \in \mathit{Agt} }, (\preceq_{i,D})_{i \in \mathit{Agt} }  ,
(\equiv_i)_{i \in \mathit{Agt} }  , ( \mathit{C}_i) _{i \in \mathit{Agt} }, V  )$
where $W$, $\preceq_{i,P}$, $\preceq_{i,D}$, $\equiv_i$, $ \mathit{C}_i$ and $V$
are as in Definition \ref{DefSemanticsModel2}
except that Constraint C4 is replaced by the weaker Constraint
C4$^*$ of Definition \ref{sec:quasiMCMdef}. 
As for MCMs, by the generated submodel property, it is easy to show that the semantics in terms of MCMCs
and the semantics in terms of quasi-MCMCs are equivalent with respect to the language $\mathcal{L}_{\mathrm{DLCAG}}$.

The only property that has to be checked carefully is whether the canonical
model for   $\mathrm{DLCAG}$ is indeed a quasi-MCMC.
To this aim,
we need to extend the proof of Lemma  \ref{canonicity}
in order 
to verify that the canonical
model
for $\mathrm{DLCAG}$
satisfies Constraint C5 of Definition \ref{DefSemanticsModel2}
and that, for all $i \in \mathit{Act}$ and for all $w \in W^c$,
$\mathit{C}_i^c(w)$ is a singleton. 
Suppose $a,b \in \mathit{C}_i^c(w)$
for $a \neq b$. The latter means that $\mathsf{play}(i,a), \mathsf{play}(i,b)\in w$.
We have $\mathsf{play}(i,a) \rightarrow \neg \mathsf{play}(i,b) \in w$,
because of Axiom \ref{ax:most}. Thus, $\neg \mathsf{play}(i,b) \in w$.
Hence, $\bot \in w$ which contradicts the fact that $w$ is a consistent theory. 
Consequently, the set $ \mathit{C}_i^c(w)$
has at most one element. Now, let us prove that 
$\mathit{C}_i^c(w)$
has at least one element. 
Because of Axiom \ref{ax:least}, we have $\bigvee_{a \in  \mathit{Act}}\mathsf{play}(i,a)  \in w$.
Thus, there exists $a \in  \mathit{Act}$ such that   $\mathsf{play}(i,a)  \in w$.
Hence, $\mathit{C}_i^c(w)$
is non-empty.
Now, let us prove that the canonical
model for   $\mathrm{DLCAG}$
satisfies Constraint C5.
Suppose
$\forall j \in \mathit{Agt}$, 
$\exists w_j \in W^c$
such that
$w \equiv_i^c w_j$  and
$\mathit{C}_j^c (w_j) = \delta(j) $.
The latter means that 
$\forall j \in \mathit{Agt}$, 
$\exists w_j \in W^c$
such that
$w \equiv_i^c w_j$  and
$ \mathsf{play}\big(j,\delta(j) \big) \in w_j$.
Thus,
we have that, $\forall j \in \mathit{Agt}$, 
$ \langle \equiv_i \rangle  \mathsf{play}\big(j,\delta(j) \big)  \in w$.
Hence, $\bigwedge_{j \in \mathit{Agt}}\langle \equiv_i \rangle  \mathsf{play}\big(j,\delta(j) \big) \in w $.
By Axiom \ref{ax:indep},
$\Big( \bigwedge_{j \in \mathit{Agt}} \langle \equiv_i \rangle  \mathsf{play}\big(j,\delta(j) \big) \Big)
 \rightarrow \langle \equiv_i \rangle  \mathsf{play}\big(\delta_{\mathit{Agt}} \big) \in w$.
 Consequently, $ \langle \equiv_i \rangle  \mathsf{play}\big(\delta_{\mathit{Agt}} \big) \in w$.
 By the existence lemma for  $\mathrm{DLCAG}$,
 the latter implies that 
 $\exists v \in W^c$ such that $w \equiv_i^c v $
 and $ \mathsf{play}\big(\delta_{\mathit{Agt}} \big) \in v$.
 Thus,
  $\exists v \in W^c$ such that $w \equiv_i^c v $
 and  $\mathit{C}_{j} (v) = \delta(j)$
 for every $j \in \mathit{Agt}$.
\end{proof}

With the support
of the language
$\mathcal{L}_{\mathrm{DLCAG}}$,
we can define a variety of notions from
the theory of games under incomplete information.
The first notion
we consider is best response,
both from
the perspective of an optimistic
agent
and from
the perspective of a pessimistic
one:
\begin{align*}
 \mathsf{BR}_i^{\mathit{Opt}}(a,\delta_{-i} ) & =_{\mathit{def}} \bigwedge_{ b \in \mathit{Act}}
\mathsf{RP}_i^{\mathit{Opt}} \Big(  \big( \mathsf{play}(i,b) \wedge  \mathsf{play}(\delta_{-i}) \big) \preceq \big( \mathsf{play}(i,a) \wedge  \mathsf{play}(\delta_{-i}) \big) \Big),\\
 \mathsf{BR}_i^{\mathit{Pess}}(a,\delta_{-i} ) & =_{\mathit{def}} 
 \bigwedge_{ b \in \mathit{Act}}
\mathsf{RP}_i^{\mathit{Pess}} \Big(  \big( \mathsf{play}(i,b) \wedge  \mathsf{play}(\delta_{-i}) \big) \preceq \big( \mathsf{play}(i,a) \wedge  \mathsf{play}(\delta_{-i}) \big) \Big).
\end{align*}
We say that playing action $a $
is for agent $i$ an optimistic (resp. pessimistic) best response to the others' joint action $\delta_{-i}$, noted $ \mathsf{BR}_i^{\mathit{Opt}}(a,\delta_{-i} ) $
(resp. $ \mathsf{BR}_i^{\mathit{Pess}}(a,\delta_{-i} ) $) if and only if for every action $b$, according to agent $i$'s optimistic
(resp. pessimistic) assessment, playing
$a$ while the others play $\delta_{-i}$ is realistically at least as good as playing
$b$ while the others play $\delta_{-i}$.

As for best response, we
can define two types of \emph{subjective}
Nash equilibrium,
one for optimistic agents
and the other for pessimistic ones.
Our notion
of subjective Nash equilibrium 
corresponds to a \emph{qualitative}
variant of the notion of Bayesian Nash equilibrium (BNE): a similar qualitative variant of BNE is studied by \cite{AmorFargier}
in the context of possibility theory.
The joint action
$\delta$
is said to be a subjective optimistic (resp. pessimistic) Nash equilibrium,
noted $\mathsf{NE}^{\mathit{Opt}}(\delta)$
(resp. $\mathsf{NE}^{\mathit{Pess}}(\delta)$), 
if 
 no  agent $i$ wants to unilaterally deviate from the chosen strategy $\delta (i)$,
 under that the assumption that $i$ is optimistic (resp. pessimistic):
\begin{align*}
\mathsf{NE}^{\mathit{Opt}}(\delta)  & =_{\mathit{def}} \bigwedge_{ i \in \mathit{Agt}}     \mathsf{BR}_i^{\mathit{Opt}}\big(\delta(i),\delta_{-i} \big) ,\\
\mathsf{NE}^{\mathit{Pess}}(\delta)  & =_{\mathit{def}} \bigwedge_{ i \in \mathit{Agt}}     \mathsf{BR}_i^{\mathit{Pess}}\big(\delta(i),\delta_{-i} \big).
\end{align*}

 Note that
assuming 
 the finiteness of
 the
 set of agents 
 $\mathit{Agt}$
 is essential
 for defining Nash equilibrium,
 since 
 our language is finitary
 and does not allow universal
 quantification
 over infinite sets.

Given the distinction
between optimistic 
and pessimistic agent,
two notions of rationality
can be defined.
Agent $i$
is said to
be optimistic (resp. pessimistic) rational, noted $\mathsf{Rat}_i^{\mathit{Opt}}$ (resp. $\mathsf{Rat}_i^{\mathit{Pess}}$), 
if she cannot choose an action that, according to her optimistic (resp. pessimistic) assessment,
is better not to choose than to choose:
\begin{align*}
\mathsf{Rat}_i^{\mathit{Opt}}  & =_{\mathit{def}} \bigwedge_{a \in \mathit{Act}}
\Big( \mathsf{play}(i,a)  \rightarrow
\mathsf{RP}_i^{\mathit{Opt}} \big( \neg \mathsf{play}(i,a)  \preceq 
\mathsf{play}(i,a) 
 \big) \Big) ,\\
\mathsf{Rat}_i^{\mathit{Pess}}  & =_{\mathit{def}} \bigwedge_{a \in \mathit{Act}}
\Big( \mathsf{play}(i,a)  \rightarrow
\mathsf{RP}_i^{\mathit{Pess}} \big( \neg \mathsf{play}(i,a)  \preceq 
\mathsf{play}(i,a) 
 \big) \Big).
\end{align*}

As the following proposition indicates,
the action chosen by
an optimistic (resp. pessimistic) rational
agent
is, 
according to the agent's optimistic (resp. pessimistic) assessment, at least as good 
as the other actions she may choose.

\begin{proposition}\label{propInterNash}
Let $i \in \mathit{Agt}$
and $x \in \{\mathit{Opt},\mathit{Pess} \}$.
Then, 
\begin{align}
\models_{\mathit{MCMC}} \big( \mathsf{Rat}_i^{x}  \wedge 
 \mathsf{play}(i,a) \big) \rightarrow
 \bigwedge_{ b \in \mathit{Act}}
 \mathsf{RP}_i^{x} \big(  \mathsf{play}(i,b)  \preceq 
\mathsf{play}(i,a) 
 \big)
\end{align}

\end{proposition}
\begin{proof}
Let us prove the case $x=\mathit{Opt}$.
Let $M$
be a MCMC and let $w$
be a world in $M$.
Suppose 
$M,w \models \mathsf{Rat}_i^{\mathit{Opt}} $
and 
$M,w \models  \mathsf{play}(i,a) $.
Thus, $M,w \models  \mathsf{RP}_i^{\mathit{Opt}} \big( \neg \mathsf{play}(i,a)  \preceq 
\mathsf{play}(i,a) 
 \big)$.
The latter means that 
$
\forall u \in  \mathit{Best}_{i,P}(w) \cap || \neg \mathsf{play}(i,a)  ||_{i,w,M},
\exists v \in  \mathit{Best}_{i,P}(w)  \cap || \mathsf{play}(i,a)  ||_{i,w,M} : u \preceq_{i,D} v.
$
Since $\models_{\mathit{MCMC}}    \mathsf{play}(i,a) \rightarrow \neg \mathsf{play}(i,b) \text{ if } a \neq b$, the latter implies 
$\forall b \in \mathit{Act},
\forall u \in  \mathit{Best}_{i,P}(w) \cap || \mathsf{play}(i,b)  ||_{i,w,M},\\
\exists v \in  \mathit{Best}_{i,P}(w)  \cap || \mathsf{play}(i,a)  ||_{i,w,M} : u \preceq_{i,D} v.
$
Thus, $ \bigwedge_{ b \in \mathit{Act}}
 \mathsf{RP}_i^{\mathit{Opt}} \big(  \mathsf{play}(i,b)  \preceq 
\mathsf{play}(i,a) 
 \big)$.
 The case $x=\mathit{Pess}$ can be proved in an analogous way.
\end{proof}


The following proposition
elucidates the connection between
the notions of belief,
rationality and
Nash equilibrium:
if all agents are optimistic (resp. pessimistic)
rational and have a correct belief
about the others' actual choices, then 
the joint action they choose
is a subjective optimistic (resp. pessimistic) Nash equilibrium.\footnote{
A similar epistemic characterization of Nash equilibrium is provided by Aumann \& Brandenburger (A\&B) \cite{AumannBrand95}
in the context of games with complete information. See also \cite{SpohnGames}
for a similar result using a probabilistic approach.}
\begin{proposition}\label{propNashchar}
Let $x \in \{\mathit{opt},\mathit{pess} \}$
and $\delta \in \mathit{JAct}$. Then:
\begin{align}
\models_{\mathit{MCMC}}
\Big( \mathsf{play}(\delta) \wedge
\bigwedge_{i \in \mathit{Agt}}\big(   \mathsf{Rat}_i^{x}  \wedge \mathsf{B}_i  \mathsf{play}(\delta_{-i}) \big)  \Big) \rightarrow \mathsf{NE}_i^{x}(\delta)   
\end{align}
\end{proposition}
\begin{proof}
Let $M$
be a MCMC and let $w$
be a world in $M$.
Suppose 
$M,w \models \mathsf{play}\big(i,\delta (i) \big) $,
$M,w \models \mathsf{Rat}_i^{x} $
and  $M,w \models \mathsf{B}_i  \mathsf{play}(\delta_{-i})  $,
for all $i \in \mathit{Agt}$.
By Proposition \ref{propInterNash}, 
it follows that
$ \bigwedge_{ b \in \mathit{Act}}
 \mathsf{RP}_i^{x} \Big(  \mathsf{play}(i,b)  \preceq 
\mathsf{play}\big(i,\delta (i) \big)  \Big)$
and  $M,w \models \mathsf{B}_i  \mathsf{play}(\delta_{-i})  $,
for all $i \in \mathit{Agt}$.
From the latter,
we can conclude that $ M,w \models \mathsf{BR}_i^{x}\big(\delta (i),\delta_{-i} \big) $,
for all $i \in \mathit{Agt}$.
Thus, $M, w \models \mathsf{NE}_i^{x}(\delta)  $.
\end{proof}


We conclude this section 
by illustrating
the game-theoretic concepts
involved
in the crossroad game
described in Section \ref{example}.
It is a game under incomplete information
since an agent does not necessarily
know the other agent's beliefs and desires.
It is single-stage since that 
interaction is non-repeated and agents are supposed
to choose simultaneously.

\smallskip

\begin{example}
Let us suppose that 
the set of actions
that 
agents $1$
and $2$
can choose
is $\mathit{Act}= \{C,S\}$,
where $C$
is the action ``to continue''
and $S$
is the action ``to stop''.
The following hypotheses
capture 
the agents' knowledge and beliefs
about actions and their effects:
\begin{align*}
\varphi_4=_{\mathit{def}} &\bigwedge_{i \in \{1,2 \}}
[\equiv_i ]  \Big(  
\big( ( \mathsf{play}(1,C) \wedge 
\mathsf{play}(2,C) ) \rightarrow\mathit{co}   \big) \wedge \\
& 
\big( ( \mathsf{play}(1,C) \wedge 
\mathsf{play}(2,S) ) \rightarrow
(\neg \mathit{lo}_1 \wedge \mathit{lo}_2 )   \big)
\wedge \\
& 
\big( ( \mathsf{play}(1,S) \wedge 
\mathsf{play}(2,C) ) \rightarrow
( \mathit{lo}_1 \wedge \neg \mathit{lo}_2 )   \big)
\wedge \\
& 
\big( ( \mathsf{play}(1,S) \wedge 
\mathsf{play}(2,S) ) \rightarrow
( \mathit{lo}_1 \wedge  \mathit{lo}_2 
\wedge \neg \mathit{co}
)
\big) \Big),\\
\varphi_5=_{\mathit{def}} &
\Big(
\big( \widehat{\mathsf{B}}_1 \mathsf{play}(1 \mapsto C,
2 \mapsto C) \leftrightarrow
\widehat{\mathsf{B}}_1 \mathsf{play} (1 \mapsto S,
2 \mapsto C) \big)
\wedge\\
&\big( \widehat{\mathsf{B}}_1\mathsf{play} (1 \mapsto C,
2 \mapsto S) \leftrightarrow
\widehat{\mathsf{B}}_1 \mathsf{play}(1 \mapsto S,
2 \mapsto S) \big)
\wedge\\
&\big( \widehat{\mathsf{B}}_2\mathsf{play} (1 \mapsto C,
2 \mapsto C) \leftrightarrow
\widehat{\mathsf{B}}_2\mathsf{play} (1 \mapsto C,
2 \mapsto S) \big)
\wedge\\
&\big( \widehat{\mathsf{B}}_2 \mathsf{play}(1 \mapsto S,
2 \mapsto C) \leftrightarrow
\widehat{\mathsf{B}}_2 \mathsf{play}(1 \mapsto S,
2 \mapsto S) \big)
\Big),
\end{align*}
where $\widehat{\mathsf{B}}_i \varphi =_{\mathit{def}}
\neg \mathsf{B}_i \neg \varphi
$.
According to the hypothesis $\varphi_4$,
the agents know that
(i)
if they both 
continue,
they will collide,
(ii)
if one of them
continues
while the other 
stops,
then the first
will lose its time
while the second will not,
and
(iii) if they both stop,
each of them
will lose its time
but there will be no collision. 
According to the
hypothesis $\varphi_5$, 
the fact
that an agent considers possible 
that the other 
will decide to continue (resp. to stop)
does not depend on the agent's choice.
This hypothesis
is justified by the assumption that
an agent's beliefs
are \emph{ex ante}, i.e.,
relative to the instant 
before
an agent makes its choice.

As the following
validity indicates,
the previous hypotheses
$\varphi_4 $ and
$ \varphi_5 $
together with the hypotheses 
$\varphi_1$ and $ \varphi_3 $
stated
in Section \ref{example}
lead to the conclusion that
(i) 
an agent's action of continuing is
both an
optimistic
and a pessimistic
best response to the other agent's
action of stopping,
and 
an agent's action of stopping is 
both an
optimistic
and a pessimistic
best response to the other agent's
action of continuing.
For every $x \in \{\mathit{opt},\mathit{pess} \}$,
we have:
\begin{align}
\models_{\mathit{MCMC}} &  
(\varphi_1  \wedge \varphi_3
\wedge
\varphi_4 \wedge \varphi_5)
\rightarrow  \notag \\
&\big( 
\mathsf{BR}_1^x
( S, 2 \mapsto C)
\wedge  
\mathsf{BR}_1^x
( C, 2 \mapsto S) \wedge
\mathsf{BR}_2^x
( S, 1 \mapsto C) \wedge
\mathsf{BR}_2^x
( C, 1 \mapsto S )\big) 
\end{align}
\end{example}

\section{Dynamic Extension}\label{sec:dynext}

The logics we presented so far merely provide a static picture of the cognitive attitudes
and choices
of multiple agents in interactive situations. 
Following the tradition 
 of dynamic epistemic logic ($\mathrm{DEL}$) \cite{kooi2007}, 
in this section
we move from
a static to a dynamic perspective and extend
the language $\mathcal{L}_{\mathrm{DLCA} }$  by a variety 
of dynamic operators
for cognitive attitude change. 
We consider two types of cognitive attitude change, namely,
radical attitude and conservative attitude change.
Radical attitude change, both in its epistemic and in its motivational form, satisfies a strong form
of success postulate. Particularly,
if an agent forms the belief that $\varphi$, as a consequence of a radical belief  revision by $\varphi$,
then
she should also form
the strong belief that $\varphi$. Analogously, 
if an agent forms the desire that $\varphi$, as a consequence of a radical desire revision by $\varphi$,
then
she should also form
the strong desire that $\varphi$.
On the contrary, after a conservative belief (resp. desire) revision by $\varphi$ is performed,
an agent may form the belief (resp. desire) that $\varphi$
without forming the strong belief (resp. strong desire)  that $\varphi$.
While radical and conservative  belief revision have been studied before
in the  literature on $\mathrm{DEL}$ \cite{BenthemRevision,BaltagSmets2008}, we are the first to apply 
 $\mathrm{DEL}$ techniques to the analysis of desire revision 
 and to oppose  belief revision
 to desire revision
 in the  $\mathrm{DEL}$ setting.\footnote{Research in the $\mathrm{DEL}$
area has rather concentrated on preference change \cite{DBLP:journals/jancl/BenthemL07,Eijck2008}, leaving desire change unexplored.   }
 
 In the rest of this section,
 we first define the semantics
 of radical belief revision
 and desire revision operators (Section \ref{sec:radchange}).
 Then, we turn to conservative attitude change
 and define the semantics
 of conservative belief revision
 and desire revision operators (Section \ref{sec:conschange}).
 Finally, we provide an axiomatics
 for the dynamic extension of our logic $\mathrm{DLCA} $ (Section \ref{sec:axchange}).


\subsection{Radical Attitude Revision}\label{sec:radchange}

Radical attitude revision operators
are of the form $[\Uparrow_{i,\tau} \varphi ] $,
with $\tau \in \{P,D\}$.
They describe the consequences of a radical
revision operation. In particular, the formula
$[\Uparrow_{i,P} \varphi ]  \psi$
is meant to stand for ``$\psi$
holds,
after agent $i$
has radically revised her beliefs with $\varphi$'',
whereas 
$[\Uparrow_{i,D} \varphi ]  \psi$
is meant to stand for ``$\psi$
holds,
after agent $i$
has radically revised her desires with $\varphi$''.
We assume that radical revision operations are public, i.e.,
if an agent radically revises her beliefs (resp. desires) with $\varphi$,
then this is common knowledge among all agents. 
This assumption could be easily relaxed by using action models as introduced in \cite{BMS1998,BaltagMossDEL}, which would allow us to model private and semi-private 
attitude change operations.
Radical revision operators
are interpreted relative to a MCM 
$M = (  W,( \preceq_{i,P})_{i \in \mathit{Agt} }, (\preceq_{i,D})_{i \in \mathit{Agt} }  ,
(\equiv_i)_{i \in \mathit{Agt} }  ,  V  )$
 and a world $w $ in $W$, as follows:
\begin{eqnarray*}
M, w \models  [\Uparrow_{i,\tau} \varphi ] \psi & \Longleftrightarrow & M^{\Uparrow_{i,\tau} \varphi }, w \models \psi ,
\end{eqnarray*}
where
\begin{align*}
M^{\Uparrow_{i,P} \varphi } = (  W,( \preceq_{i,P}^{\Uparrow_{i,P}  \varphi })_{i \in \mathit{Agt} }, (\preceq_{i,D})_{i \in \mathit{Agt} }  ,
(\equiv_i)_{i \in \mathit{Agt} }  ,   V  ),
\end{align*}
\begin{align*}
M^{\Uparrow_{i,D} \varphi } = (  W,( \preceq_{i,P})_{i \in \mathit{Agt} }, (\preceq_{i,D}^{\Uparrow_{i,D}  \varphi })_{i \in \mathit{Agt} }  ,
(\equiv_i)_{i \in \mathit{Agt} }  ,   V  ),
\end{align*}
\begin{align*}
\preceq_{i,\tau}^{\Uparrow_{i,\tau}  \varphi }=  \Big\{ (w,v) \in W \times W : & \big( ( M,w \models \varphi \text{ iff }  M,v \models \varphi ) \text{ and } w\preceq_{i,\tau} v  \big)
\text{ or } \\
& ( M,w \models\neg \varphi \text{, }  M,v \models \varphi \text{ and } w \equiv_i v )   \Big \},
\end{align*}
and 
$\preceq_{j,\tau}^{\Uparrow_{i,\tau}  \varphi }=\preceq_{j,\tau} \text{ for all } j  \in \mathit{Agt} \text{ such that } i \neq j$.

Radical belief
and desire revision are 
completely
symmetric
from the point of view of
the plausibility
and desirability ordering.
Agent $i$'s
radical belief revision with $\varphi$
transforms agent $i$'s plausibility ordering
$\preceq_{i,P}$
into the new plausibility ordering
$\preceq_{i,P}^{\Uparrow_{i,P}  \varphi }$.
 In particular, it makes
all $\varphi$-worlds in $i$'s information set more plausible  than all $\neg \varphi$-worlds
and, within those two zones, it keeps the old plausibility ordering. 
Analogously, 
agent $i$'s
radical desire revision with $\varphi$
transforms agent $i$'s desirability ordering
$\preceq_{i,D}$
into the new desirability ordering
$\preceq_{i,D}^{\Uparrow_{i,D}  \varphi }$.
It
makes
all $\varphi$-worlds in $i$'s information set more desirable  than all $\neg \varphi$-worlds
and, within those two zones, it keeps the old desirability ordering.

As emphasized above, 
radical revision satisfies a strong form
of success principle which is formally expressed by the following two validities.
 Let  $\varphi\in \mathcal{L}_{\mathsf{PL}}(\mathit{Atm})$. Then,
\begin{align}
& \models_{\mathit{MCM}}   \langle \equiv_i \rangle \varphi \rightarrow [\Uparrow_{i,P} \varphi ] (  \mathsf{B}_i \varphi \wedge \mathsf{SB}_i \varphi) \\
&\models_{\mathit{MCM}}   \langle \equiv_i \rangle \neg \varphi \rightarrow [\Uparrow_{i,D} \varphi ] (  \mathsf{D}_i \varphi \wedge \mathsf{SD}_i \varphi)
\label{val1}
\end{align}
This means that (i) if $\varphi$
is compatible with an agent's knowledge then,
after she has radically revised her beliefs  with $\varphi$,
the agent will both believe  that $\varphi$
and strongly believe  that $\varphi$,
and
(ii) 
if $\neg \varphi$
is compatible with an agent's knowledge then,
after she has radically revised her  desires with $\varphi$,
the agent will both  desire that $\varphi$
and  strongly desire that $\varphi$.
 
The 
two 
 validities
 highlight  that 
 belief
 and desire behave in a slightly different way under
 radical revision, despite
the fact that
 the plausibility
and desirability ordering
are modified in the same way.

We  have 
 the following additional validities, for $\varphi \in \mathcal{L}_{\mathsf{PL}}(\mathit{Atm})$:
\begin{align}
& \models_{\mathit{MCM}}   [\Uparrow_{i,P} \varphi ] (  \mathsf{B}_i \varphi  \rightarrow \mathsf{SB}_i \varphi )  \\
& \models_{\mathit{MCM}}   [\Uparrow_{i,D} \varphi ] (  \mathsf{D}_i \varphi  \rightarrow \mathsf{SD}_i \varphi )  
\end{align}
This means that the formation of a belief (resp. desire) through
radical belief (resp. desire) revision
  necessarily entails
the formation of a strong belief (resp. strong desire)
with the same content.

\smallskip

\begin{example}
Let us go back to the
crossroad game
in order 
to illustrate the 
radical
desire revision mechanism.
Suppose
agent $1$
performs
a radical
desire revision
operation with 
$\neg \mathit{lo}_2$,
since it learns
that 
agent $2$
is an ambulance which 
has to
lose no time 
at the crossroad.
By the previous validity
(\ref{val1}),
we can prove that,
under the hypothesis $\varphi_2$ stated in Section \ref{example},
$1$
will both desire 
and strongly desire that $\neg \mathit{lo}_2$,
after the radical desire revision operation
with $\neg \mathit{lo}_2$:
\begin{align}
& \models_{\mathit{MCM}}  
 \varphi_2 \rightarrow
 [\Uparrow_{1,D} \neg \mathit{lo}_2 ] (  \mathsf{D}_1 \neg \mathit{lo}_2  \wedge \mathsf{SD}_1 \neg \mathit{lo}_2 ) 
\end{align}
Moreover,
under the set
of hypotheses $\{\varphi_1, \varphi_2, \varphi_3 \}$,
after the radical desire revision operation
with $\neg \mathit{lo}_2$,
$1$
will  not strongly
desire
anymore not to lose time,
but it will merely desire it:
\begin{align}
& \models_{\mathit{MCM}}  
 (\varphi_1 \wedge \varphi_2 \wedge
 \varphi_3) \rightarrow
 [\Uparrow_{1,D} \neg \mathit{lo}_2 ] (   \mathsf{D}_1 \neg \mathit{lo}_1  \wedge
 \neg \mathsf{SD}_1 \neg \mathit{lo}_1 ) 
\end{align}
\end{example}

As the following proposition indicates,
we have reduction axioms
which allow us to eliminate radical attitude revision operators
from a formula.
\begin{proposition}\label{propred1}
The following equivalences are  valid:
\begin{align*}
& [\Uparrow_{i,\tau} \varphi ] p \leftrightarrow p \\
& [\Uparrow_{i,\tau} \varphi ] x \leftrightarrow x \\
& [\Uparrow_{i,\tau} \varphi ] \neg \psi \leftrightarrow \neg[\Uparrow_{i,\tau} \varphi ] \psi  \\
& [\Uparrow_{i,\tau} \varphi ] ( \psi_1 \wedge \psi_2) \leftrightarrow ( [\Uparrow_{i,\tau} \varphi ]  \psi_1 \wedge 
 [\Uparrow_{i,\tau} \varphi ]  \psi_2 )  \\
 &[\Uparrow_{i,\tau} \varphi ]  [\pi]  \psi \leftrightarrow [F^{\Uparrow_{i,\tau} \varphi } (\pi)] [\Uparrow_{i,\tau} \varphi ] \psi \\
\end{align*}
where for all $j \in \mathit{Agt}$
and for all $\tau, \tau' \in \{P,D\}$:
\begin{align*}
&F^{\Uparrow_{i,\tau} \varphi } ( \equiv_j) =  \equiv_j \\
&F^{\Uparrow_{i,\tau} \varphi } ( \preceq_{i,\tau}) = (  \varphi? ;  \preceq_{i,\tau}; \varphi?  ) \cup
 (  \neg \varphi? ;  \preceq_{i,\tau}; \neg \varphi ? ) \cup  (  \neg \varphi? ;  \equiv_{i};  \varphi? )  \\
 &F^{\Uparrow_{i,\tau} \varphi } ( \preceq_{j,\tau'}) = \preceq_{j,\tau'} \text{ if }  i \neq j \text{ or } \tau \neq \tau'\\
&F^{\Uparrow_{i,\tau} \varphi } (\preceq_{i,\tau}^\sim) = (  \varphi? ;  \preceq_{i,\tau}^\sim; \varphi?  ) \cup
 (  \neg \varphi? ;  \preceq_{i,\tau}^\sim; \neg \varphi ?) \cup  (   \varphi? ;  \equiv_{i}; \neg \varphi?) \\
&F^{\Uparrow_{i,\tau} \varphi } (\preceq_{j,\tau'}^\sim) =\preceq_{j,\tau'}^\sim \text{ if }  i \neq j \text{ or } \tau \neq \tau'\\
&F^{\Uparrow_{i,\tau} \varphi } (  \pi { ; } \pi' ) = F^{\Uparrow_{i,\tau} \varphi } (  \pi  ){ ; } F^{\Uparrow_{i,\tau} \varphi } (   \pi' )  \\
&F^{\Uparrow_{i,\tau} \varphi } (  \pi \cup \pi' ) = F^{\Uparrow_{i,\tau} \varphi } (  \pi  )\cup F^{\Uparrow_{i,\tau} \varphi } (   \pi' )  \\
&F^{\Uparrow_{i,\tau} \varphi } (  \pi \cap \pi' ) = F^{\Uparrow_{i,\tau} \varphi } (  \pi  )\cap F^{\Uparrow_{i,\tau} \varphi } (   \pi' )  \\
&F^{\Uparrow_{i,\tau} \varphi } ( -\pi  ) =- F^{\Uparrow_{i,\tau} \varphi } ( \pi  ) \\
&F^{\Uparrow_{i,\tau} \varphi } (  \psi?  ) = [\Uparrow_{i,\tau} \varphi ]\psi ?
\end{align*}

\end{proposition}

\subsection{Conservative Attitude Revision}\label{sec:conschange}

Let us move from
radical attitude change
to conservative attitude change
by introducing radical
revision operators
of type $  [\uparrow_{i,\tau} \varphi ]$,
with $\tau \in \{P,D\}$.
The formula
$[\uparrow_{i,P} \varphi ]  \psi$
(resp. $[\uparrow_{i,D} \varphi ]  \psi$)
is meant to stand for ``$\psi$
holds,
after agent $i$
has conservatively revised her beliefs (resp. desires) with $\varphi$''.
As for radical revision,
we assume that conservative revision operations are public, i.e.,
if an agent conservatively revises her beliefs (resp. desires) with $\varphi$,
then this is common knowledge among all agents. 
The semantic
interpretation of such operators relative
to a MCM
$M = (  W,( \preceq_{i,P})_{i \in \mathit{Agt} }, (\preceq_{i,D})_{i \in \mathit{Agt} }  ,
(\equiv_i)_{i \in \mathit{Agt} }  ,  V  )$
be a MCM and a world $w $
in $W$ is as follows:
\begin{eqnarray*}
M, w \models  [\uparrow_{i,\tau} \varphi ] \psi & \Longleftrightarrow & M^{\uparrow_{i,\tau} \varphi }, w \models \psi ,
\end{eqnarray*}
where:
\begin{align*}
M^{\uparrow_{i,P} \varphi } = (  W,( \preceq_{i,P}^{\uparrow_{i,P}  \varphi })_{i \in \mathit{Agt} }, (\preceq_{i,D})_{i \in \mathit{Agt} }  ,
(\equiv_i)_{i \in \mathit{Agt} }  ,   V  ), \\
M^{\uparrow_{i,D} \varphi } = (  W,( \preceq_{i,P})_{i \in \mathit{Agt} }, (\preceq_{i,D}^{\uparrow_{i,D}  \varphi })_{i \in \mathit{Agt} }  ,
(\equiv_i)_{i \in \mathit{Agt} }  ,   V  ),
\end{align*}
with:
\begin{align*}
\preceq_{i,P}^{\uparrow_{i,P}  \varphi }=  \Big\{ (w,v) \in W \times W : & \Big( \big( 
w  \in \mathit{Best}_{i,P}(\varphi, w) \text{ iff }  v  \in \mathit{Best}_{i,P}(\varphi, w) \big) \text{ and } w\preceq_{i,P} v  \Big)
\text{ or } \\
& \big( w \not \in \mathit{Best}_{i,P}(\varphi, w),
v  \in \mathit{Best}_{i,P}(\varphi, w)
 \text{ and } w \equiv_i v \big)   \Big \},\\
\preceq_{i,D}^{\uparrow_{i,D}  \varphi }=  \Big\{ (w,v) \in W \times W : & \Big( \big( 
w  \in \mathit{Worst}_{i,D}(\neg \varphi, w) \text{ iff }  v  \in \mathit{Worst}_{i,D}(\neg \varphi, w) \big) \text{ and } w\preceq_{i,D} v  \Big)
\text{ or } \\
& \big( w  \in \mathit{Worst}_{i,D}(\neg \varphi, w),
v  \not \in \mathit{Worst}_{i,D}( \neg \varphi, w)
 \text{ and } w \equiv_i v \big)   \Big \},
\end{align*}
and 
$\preceq_{j,\tau}^{\uparrow_{i,\tau}  \varphi }=\preceq_{j,\tau} \text{ for all } j  \in \mathit{Agt} \text{ such that } i \neq j$.

Unlike 
radical  revision,
plausibility update
and desirability
update
 are
asymmetric under conservative revision.
Agent $i$'s
conservative belief revision with $\varphi$
replaces the current plausibility ordering
$\preceq_{i,P}$
with the new plausibility ordering
$\preceq_{i,P}^{\uparrow_{i,P}  \varphi }$.
It promotes the most plausible  $\varphi$-worlds
 to the highest plausibility rank, but apart from that, the
 old plausibility ordering remains.
 Agent $i$'s
conservative desire revision with $\varphi$
replaces the current desirability ordering
$\preceq_{i,D}$
with the new desirability ordering
$\preceq_{i,P}^{\uparrow_{i,P}  \varphi }$.
 In particular, it demotes the least desirable $\neg \varphi$-worlds
 to the lowest desirability rank, but apart from that, the
 old desirability  ordering remains.

Conservative attitude revision satisfies a weak form
of success principle which guarantees the formation
of a belief (resp. a desire), after a belief (resp. desire)
revision is performed.
 Let  $\varphi\mathcal{L}_{\mathsf{PL}}(\mathit{Atm})$. Then,
\begin{align}
& \models_{\mathit{MCM}} \neg \mathsf{B}_i (\varphi {,} \bot) \rightarrow [\uparrow_{i,P} \varphi ]   \mathsf{B}_i \varphi \label{valusef2} \\
&\models_{\mathit{MCM}}  \neg \mathsf{D}_i (\varphi, \top)  \rightarrow [\uparrow_{i,D} \varphi ]   \mathsf{D}_i \varphi 
\end{align}
According to the previous validities, if 
an agent does not believe a contradiction conditional on $\varphi$ then,
after she has conservatively revised her beliefs  with $\varphi$,
she  will believe that $\varphi$.
 If
an agent does not desire a tautology conditional on $\varphi$ then,
after she has conservatively revised her desires with $\varphi$,
she  will desire that $\varphi$.
But, unlike radical attitude revision,
conservative attitude revision
does not necessarily guarantee the   formation
of a strong belief (resp. a strong desire), after a belief (resp. desire)
revision is performed. Indeed,
we have the following, for $\varphi \in \mathcal{L}_{\mathsf{PL}}(\mathit{Atm})$:
\begin{align}
&\not \models_{\mathit{MCM}}   [\uparrow_{i,P} \varphi ] (  \mathsf{B}_i \varphi  \rightarrow \mathsf{SB}_i \varphi )  \\
&\not \models_{\mathit{MCM}}   [\uparrow_{i,D} \varphi ] (  \mathsf{D}_i \varphi  \rightarrow \mathsf{SD}_i \varphi )  
\end{align}
This means that the formation of a belief (resp. desire) through
conservative belief (resp. desire) revision
does not necessarily entail
the formation of a strong belief (resp. strong desire)
with the same content.

\begin{example}
Let us illustrate the 
conservative 
belief revision mechanism
with the help
of the crossroad game.
Suppose
agent $1$
informs agent $2$
that 
``if they both
lose time,
then there will no collision''
and, as a consequence,
$2$ performs
a 
conservative
belief
revision
operation with input $(\mathit{lo}_1 \wedge
\mathit{lo}_2 )\rightarrow \neg \mathit{co} $.
By the previous validity
(\ref{valusef2})
we can prove that,
under the hypothesis $\varphi_1$ stated in Section \ref{example}
and the assumption that
$2$ does
not believe
a contradiction
conditional on $1$'s assertion,
$2$
 believes
 that there will be no collision,
 after its conservative belief operation:
\begin{align}
& \models_{\mathit{MCM}}  
 \Big( 
\neg  \mathsf{B}_2 \big(
 (\mathit{lo}_1 \wedge
\mathit{lo}_2 )\rightarrow \neg \mathit{co}
  {,} \bot\big) \wedge
 \varphi_1 \Big) \rightarrow
 [\uparrow_{2,P}(\mathit{lo}_1 \wedge
\mathit{lo}_2 )\rightarrow \neg \mathit{co}  ]  \mathsf{B}_2 \neg \mathit{co}  
\end{align}
\end{example}

As for radical revision,
we have reduction axioms
which allow us to eliminate conservative attitude revision operators
from a formula.
\begin{proposition}\label{propred2}
The following equivalences are  valid:
\begin{align*}
& [\uparrow_{i,\tau} \varphi ] p \leftrightarrow p \\
& [\uparrow_{i,\tau} \varphi ] x \leftrightarrow x \\
& [\uparrow_{i,\tau} \varphi ] \neg \psi \leftrightarrow \neg[\uparrow_{i,\tau} \varphi ] \psi  \\
& [\uparrow_{i,\tau} \varphi ] ( \psi_1 \wedge \psi_2) \leftrightarrow ( [\uparrow_{i,\tau} \varphi ]  \psi_1 \wedge 
 [\uparrow_{i,\tau} \varphi ]  \psi_2 )  \\
 &[\uparrow_{i,\tau} \varphi ]  [\pi]  \psi \leftrightarrow [F^{\uparrow_{i,\tau} \varphi } (\pi)] [\uparrow_{i,\tau} \varphi ] \psi \\
\end{align*}
where for all $j \in \mathit{Agt}$
and for all $\tau, \tau' \in \{P,D\}$:
\begin{align*}
F^{\uparrow_{i,\tau} \varphi } ( \equiv_j) =&  \equiv_j \\
F^{\uparrow_{i,P} \varphi } ( \preceq_{i,P}) =& \big(   (\varphi \wedge [\prec_{i,P}]\neg \varphi)? ;  \preceq_{i,\tau};  (\varphi \wedge [\prec_{i,P}]\neg \varphi)?  \big) \cup\\
& \big(  \neg (\varphi \wedge [\prec_{i,P}]\neg \varphi)? ;  \preceq_{i,\tau}; \neg (\varphi \wedge [\prec_{i,P}]\neg \varphi) ? \big) \cup\\
&  \big(  \neg (\varphi \wedge [\prec_{i,P}]\neg \varphi)? ;  \equiv_{i};  (\varphi \wedge [\prec_{i,P}]\neg \varphi)? \big)  \\
F^{\uparrow_{i,D} \varphi } ( \preceq_{i,D}) =& \big(   (\neg \varphi \wedge [\succ_{i,D}]\varphi)? ;  \preceq_{i,\tau};  (\neg \varphi \wedge [\succ_{i,D}]\varphi)?  \big) \cup\\
& \big(  \neg (\neg \varphi \wedge [\succ_{i,D}]\varphi)? ;  \preceq_{i,\tau}; \neg (\neg \varphi \wedge [\succ_{i,D}]\varphi) ? \big) \cup\\
&  \big(  (\neg \varphi \wedge [\succ_{i,D}]\varphi)? ;  \equiv_{i}; \neg (\neg \varphi \wedge [\succ_{i,D}]\varphi)? \big)  \\
 F^{\uparrow_{i,\tau} \varphi } ( \preceq_{j,\tau'}) = &\preceq_{j,\tau'} \text{ if }  i \neq j \text{ or } \tau \neq \tau'\\
F^{\uparrow_{i,P} \varphi } (\preceq_{i,P}^\sim) =& \big(   (\varphi \wedge [\prec_{i,P}]\neg \varphi)? ;  \preceq_{i,\tau};  (\varphi \wedge [\prec_{i,P}]\neg \varphi)?  \big) \cup\\
& \big(  \neg (\varphi \wedge [\prec_{i,P}]\neg \varphi)? ;  \preceq_{i,\tau}; \neg (\varphi \wedge [\prec_{i,P}]\neg \varphi) ? \big) \cup\\
&  \big(   (\varphi \wedge [\prec_{i,P}]\neg \varphi)? ;  \equiv_{i}; \neg  (\varphi \wedge [\prec_{i,P}]\neg \varphi)? \big)  \\
 F^{\uparrow_{i,D} \varphi } (\preceq_{i,D}^\sim) = & \big(   (\neg \varphi \wedge [\succ_{i,D}]\varphi)? ;  \preceq_{i,\tau};  (\neg \varphi \wedge [\succ_{i,D}]\varphi)?  \big) \cup\\
& \big(  \neg (\neg \varphi \wedge [\succ_{i,D}]\varphi)? ;  \preceq_{i,\tau}; \neg (\neg \varphi \wedge [\succ_{i,D}]\varphi) ? \big) \cup\\
&  \big( \neg  (\neg \varphi \wedge [\succ_{i,D}]\varphi)? ;  \equiv_{i}; (\neg \varphi \wedge [\succ_{i,D}]\varphi)? \big)  \\
F^{\uparrow_{i,\tau} \varphi } (\preceq_{j,\tau'}^\sim) =&\preceq_{j,\tau'}^\sim \text{ if }  i \neq j \text{ or } \tau \neq \tau'\\
F^{\uparrow_{i,\tau} \varphi } (  \pi { ; } \pi' ) = &F^{\uparrow_{i,\tau} \varphi } (  \pi  ){ ; } F^{\uparrow_{i,\tau} \varphi } (   \pi' )  \\
F^{\uparrow_{i,\tau} \varphi } (  \pi \cup \pi' ) =& F^{\uparrow_{i,\tau} \varphi } (  \pi  )\cup F^{\uparrow_{i,\tau} \varphi } (   \pi' )  \\
F^{\uparrow_{i,\tau} \varphi } (  \pi \cap \pi' ) =& F^{\uparrow_{i,\tau} \varphi } (  \pi  )\cap F^{\uparrow_{i,\tau} \varphi } (   \pi' )  \\
F^{\uparrow_{i,\tau} \varphi } ( -\pi  ) =& - F^{\uparrow_{i,\tau} \varphi } ( \pi  ) \\
F^{\uparrow_{i,\tau} \varphi } (  \psi?  ) =& [\uparrow_{i,\tau} \varphi ]\psi ?
\end{align*}

\end{proposition}

\subsection{Dynamic Logic of Cognitive Attitudes and their Change }\label{sec:axchange}

The modal
language $\mathcal{L}_{\mathrm{DLCAC} } ( \mathit{Atm},  \mathit{Nom},  \mathit{Agt})$,
or simply $\mathcal{L}_{\mathrm{DLCAC} }$,
for the Dynamic Logic of Cognitive Attitudes and their Change
 ($\mathrm{DLCAC}$)
 extends the language 
  $\mathcal{L}_{\mathrm{DLCA} }$ 
  of the logic   $\mathrm{DLCA}$
  by dynamic operators
  of type $[\Uparrow_{i,\tau} \varphi ]$
  and
$  [\uparrow_{i,\tau} \varphi ]$.
It is defined by the following grammar:
\begin{center}\begin{tabular}{lcl}
 $\varphi$  & $::=$ & $ p \mid
 x
 \mid \neg\varphi \mid \varphi\wedge\varphi' \mid [\pi] \varphi \mid   [\Uparrow_{i,\tau} \varphi ] \psi \mid   [\uparrow_{i,\tau} \varphi ] \psi
$
\end{tabular}\end{center}
where
$\pi$
ranges over the language of cognitive programs $\mathcal{P}$,
 $p$ ranges over $ \mathit{Atm} $,
$x$ ranges over $ \mathit{Nom} $,
$i$ ranges over $ \mathit{Agt} $
and $\tau $
ranges over $\{P,D\}$.

\begin{definition}
We define  $\mathrm{DLCAC}$
to be the extension of $\mathrm{DLCA}$ given by the
reduction principles of Proposition \ref{propred1} and Proposition \ref{propred2}
and the following 
rule of replacement of equivalents
\begin{align}
& \frac{\psi_1 \leftrightarrow \psi_2 }{ \varphi \leftrightarrow \varphi[\psi_1 / \psi_2]}
                                           \tagLabel{REP}{ax:RepRule}
\end{align}
where $\varphi[\psi_1 / \psi_2]$ is the formula that results from $\varphi$ by replacing zero or more occurrences of $\psi_1$, in $\varphi$, by $\psi_2$.
\end{definition}
As the rule of replacement of equivalents  preserves validity,
the equivalences of  Propositions \ref{propred1} and \ref{propred2} 
together with this
allow to reduce every formula of the language $\mathcal{L}_{\mathrm{DLCAC} }$ to an equivalent formula
of the language $\mathcal{L}_{\mathrm{DLCA} }$.
Call $\mathit{red}$ the mapping which iteratively applies the above equivalences
from the left to the right, starting from one of the innermost modal operators.
$\mathit{red}$ pushes the dynamic operators inside the formula, and finally eliminates them when
facing an atomic formula.

\begin{proposition}\label{theo:reductionProp}
Let $\varphi$ be a formula in the language of $\mathcal{L}_{\mathrm{DLCAC} }$. Then
\begin{itemize}
\item
$\mathit{red}(\varphi) $ has no dynamic operators $[\Uparrow_{i,\tau} \varphi ]$
  or
$  [\uparrow_{i,\tau} \varphi ]$, and
\item
$\mathit{red}(\varphi) \leftrightarrow \varphi$ is valid
relative to the class of MCMs.
\end{itemize}
\end{proposition}
The first item of Proposition \ref{theo:reductionProp} is clear.
The second item is proved using  the equivalences of  Propositions \ref{propred1} and \ref{propred2}
 and the rule of replacement of equivalents.

 The following theorem is a direct consequence of 
 Theorem \ref{firstcomplete} and 
 Proposition \ref{theo:reductionProp}.
 
\begin{theorem}\label{compldynamic}
The logic $\mathrm{DLCAC}$ is sound and complete for the class of multi-agent cognitive models.
\end{theorem}

\section{Conclusion and perspectives}\label{ConclusionPart}
  
We have presented a 
logical framework
for  modelling
a rich variety of
cognitive attitudes
of  both epistemic
type
and motivational type.
We have presented two extensions of the basic setting,
one by the notion of choice
and the other by dynamic operators for belief change and desire change.  
We have applied the former to the analysis of games
under incomplete information.
We have provided sound
and complete
axiomatizations for the basic setting
and for its two extensions. 
Directions of future research
are manifold
and are briefly discussed
in the rest of this section.

\paragraph{Decidability and complexity}
The present paper is  devoted to
study
the proof-theoretic
aspects
of the proposed logics.
In future work,
we plan to investigate
their
computational
aspects
including decidability
of their satisfiability
problems and, at  a later stage,
complexity. 
In order to prove decidability,
we expect to be able
to use
existing 
 filtration techniques from
 modal logic. 
 Note that once we have proved decidability of
 the static setting $\mathrm{DLCA}$, we can use the reduction 
 axioms
 of 
 Propositions \ref{propred1} and \ref{propred2} 
 to prove decidability of the dynamic setting $\mathrm{DLCAC}$.

We plan to study
 complexity
 of the satisfiability
 problems
 for interesting fragments
 of the language 
 $\mathcal{L}_\mathrm{DLCA}$
 by reducing them
 to satisfiability
 problems
 of existing logics.
 For instance, consider 
 the following  
 single-agent ($\mathit{sa}$) fragment
 of the language
  $\mathcal{L}_\mathrm{DLCA}$
  where only atomic programs
  ($\mathit{ap}$)
   are allowed,
  noted $\mathcal{L}_\mathrm{DLCA}^{ \mathit{sa},
  \mathit{ap} }$:
  \begin{center}\begin{tabular}{lcl}
 $\varphi$  & $::=$ & $ p \mid x\mid
 \neg\varphi \mid \varphi\wedge\varphi' \mid [\preceq_{1,P}] \varphi \mid  [\preceq_{1,D}] \varphi \mid [\equiv_{1}] \varphi
$
\end{tabular}\end{center}
where $1$
is an arbitrary agent in $\mathit{Agt}$.
We can observe that
 the satisfiability
 problem
 for this fragment
  is EXPTIME-hard.
 Indeed,
 because of Constraint C1
 in Definition 
 \ref{DefSemanticsModel},
 the modality
 $[\equiv_{1} ]$
 plays the role of
 the universal modality
 with respect
 to the  modalities 
 $[\preceq_{1,P}]$
 and $[\preceq_{1,D}]$.
 As shown in 
 \cite{DBLP:journals/ndjfl/Hemaspaandra96},
 adding the universal modality
 to a multimodal logic
 with independent modalities,
 such as the S4-modalities $[\preceq_{1,P}]$ and $[\preceq_{1,D}]$, causes EXPTIME-hardness.

 Consider now
 the following 
 intersection-free ($\mathit{if}$)
 and complement-free ($\mathit{cf}$)
 fragment
 of
   $\mathcal{L}_\mathrm{DLCA}$,
   noted $\mathcal{L}_\mathrm{DLCA}^{\mathit{if},
   \mathit{cf}}$:
 \begin{center}\begin{tabular}{lcl}
  $\pi$  & $::=$ & $    \equiv_i \mid \preceq_{i,P} \mid \preceq_{i,D}    \mid 
   \pi { ; } \pi'
  \mid
  \pi \cup \pi'   \mid
  -\pi  \mid \varphi?
   $\\
 $\varphi$  & $::=$ & $ p \mid
 x
 \mid \neg\varphi \mid \varphi\wedge\varphi' \mid [\pi] \varphi
$
\end{tabular}\end{center}
Our first conjecture
is
that we can find a polysize
reduction
of
the satisfiability
problem
for $\mathcal{L}_\mathrm{DLCA}^{\mathit{if},
   \mathit{cf}}$
to the satisfiability
problem
of 
 converse propositional dynamic logic ($\mathrm{PDL}$) with nominals,
 also called
 converse combinatory propositional dynamic logic  ($\mathrm{CcPDL}$).\footnote{
 The main idea of the polynomial embedding
 is to exploit 
 the 
 iteration construct $^*$
 of $\mathrm{PDL}$
 for the translation 
 $\mathit{tr}$
 of the
 cognitive programs,
 by stipulating that
$\mathit{tr}(\equiv_i)=(\mathit{any}_i \cup  -\mathit{any}_i)^*$,
  $\mathit{tr}(\preceq_{i,P})= P_i^*$,
   $\mathit{tr}(\preceq_{i,D})=D_i^*$,
   and homomorphic otherwise,
 where
 $\mathcal{A}_i$
 is agent $i$'s set of atomic
 programs (or actions), 
  $\mathcal{A}= 
  \bigcup_{
i \in 
 \mathit{Agt}}
   \mathcal{A}_i$
   is the set of
   $\mathrm{PDL}$
    atomic
   programs,
  $\mathit{any}_i= \bigcup_{
 a_i \in 
 \mathcal{A}_i}$ and, finally,
 $P_i^*$
 	and $D_i^*$
 	are special atomic programs
 	in $\mathcal{A}_i$.} The latter problem is known
 to be EXPTIME-complete  \cite{PhdDeGiacomo}.
  Therefore,
 if our conjecture is true,
 we will be able to conclude 
 that
the satisfiability
problems
for the fragments
$\mathcal{L}_\mathrm{DLCA}^{\mathit{sa},
   \mathit{ap}}$ and $\mathcal{L}_\mathrm{DLCA}^{\mathit{if},
   \mathit{cf}}$
are both EXPTIME-complete. 

We also
intend to study complexity of the nominal-free ($\mathit{nf}$)
fragment of
$\mathcal{L}_\mathrm{DLCA}$,
noted $\mathcal{L}_\mathrm{DLCA}^{\mathit{nf}}$.
 Nominals play
a technical
role in the logic $\mathrm{DLCA}$
by making it easier the task
of axiomatizing 
intersection
and complement
of programs
(Axioms \ref{ax:Prg3} and \ref{ax:Prg7} in Definition \ref{axiomatics}).
Our second  conjecture
is  that the language $\mathcal{L}_{\mathrm{DLCA} }$
is strictly more expressive than its nominal-free
fragment $\mathcal{L}_\mathrm{DLCA}^{\mathit{nf}}$.
Our third conjecture is
that we can find
a polysize reduction
of the satisfiability
problem for $\mathcal{L}_\mathrm{DLCA}^{\mathit{nf}}$
to
the satisfiability problem
of boolean modal logic
with a bounded number of modal parameters
 which is known to be EXPTIME-complete
 \cite{DBLP:conf/aiml/LutzS00}. 
 We leave the proof of the previous three conjectures to future work.
 We 
 leave to future work
 (i)
 the proof of the previous three conjectures,
 and (ii)
 the development of
  tableau-based
automated reasoning procedures
for 
the language
$\mathcal{L}_\mathrm{DLCA}$
and for its fragments 
 $\mathcal{L}_\mathrm{DLCA}^{ \mathit{sa},
  \mathit{ap} }$,
$\mathcal{L}_\mathrm{DLCA}^{\mathit{if},
   \mathit{cf}}$
   and
$\mathcal{L}_\mathrm{DLCA}^{\mathit{nf}}$
 which can be used for programming artificial agents endowed with cognitive attitudes.

 \paragraph{Well-foundedness}
 Future work
 will also
 be devoted to study 
 a variant of our logic $\mathrm{DLCA}$
 under the assumption of
 converse well-foundedness
for the relation 
$ \preceq_{i,P}$
and well-foundedness
for the relation 
$ \preceq_{i,D}$.
As emphasized in Section \ref{FormalizationPart},
these properties are required to make
agents' beliefs
and desires consistent, namely,
to guarantee that 
the formulas
$\neg(\mathsf{B}_i \varphi \wedge \mathsf{B}_i \neg \varphi) $,
$\neg\mathsf{B}_i \bot $,
$\neg(\mathsf{D}_i \varphi \wedge \mathsf{D}_i \neg \varphi )$
and $\neg\mathsf{D}_i \top$
become valid.
We will define 
the logic 
$\mathrm{DLCA}^{\mathit{wf}}$ 
to be the extension 
of the logic $\mathrm{DLCA}$  of Definition
\ref{axiomatics}
by the following two axioms:
\begin{align}
& \langle \equiv_i  \rangle \psi \rightarrow
\langle \equiv_i  \rangle (\psi \wedge [\prec_{i,P}]\neg \psi)
 \tagLabel{CWF$_{ \preceq_{i,P}}$}{ax:ConvWF1}\\
 &\langle \equiv_i  \rangle \psi \rightarrow
\langle \equiv_i  \rangle (\psi \wedge [\succ_{i,D}]\neg \psi)
 \tagLabel{WF$_{ \preceq_{i,D}}$}{ax:ConvWF2}
\end{align}
Such axioms are variants of the so-called
G\"{o}del-L\"{o}b ($\mathbf{GL}$) axiom from provability logic \cite{BoolosProvability}.
Our conjecture 
is that
the logic
$\mathrm{DLCA}^{\mathit{wf}}$
so defined is sound and complete 
 for the class of
multi-agent cognitive models (MCMs)
whose relations 
$ \preceq_{i,D}$
and
$ \preceq_{i,P}$
are, respectively,
well-founded and conversely well-founded.
 

\paragraph{Ceteris paribus preference}
We also plan to study 
a \emph{ceteris paribus}
notion of dyadic  preference
in the sense of Von Wright \cite{WrightPreference},
which has been recently formalized in a modal logic
setting by van Benthem et al.  \cite{BenthemGirardRoy}.
According to Von Wright,
for an agent 
to have
a preference of $\varphi$ over $\psi$,
she should prefer a situation in which $\varphi$ is true
     to a situation in which $\psi$
     is true,    \emph{all other things being equal}.\footnote{See also   \cite{WellmanDoyle} for a ``ceteris paribus'' interpretation
   of the notion of goal.}
     Our aim is to show that the
      $\mathrm{DLCA}$
framework is expressive enough to capture both
the static and the dynamic aspects of
this notion of 
 \emph{ceteris paribus}  preference.

\section*{Acknowledgments}

This work was supported
by the ANR project CoPains (``Cognitive Planning in Persuasive Multimodal Communication''). 
Support from the ANR-3IA Artificial and Natural Intelligence Toulouse Institute is also gratefully acknowledged.


\begin{thebibliography}{10}

\bibitem{Alc85}
C.~E. Alchourr{\'{o}}n, P.~Gardenfors, and D.~Makinson.
\newblock On the logic of theory change: Partial meet contraction and revision
  functions.
\newblock {\em The Journal of Symbolic Logic}, 50:510--530, 1985.

\bibitem{AmorFargier}
N.~B. Amor, H.~Fargier, R.~Sabbadin, and M.~Trabelsi.
\newblock Possibilistic games with incomplete information.
\newblock In {\em Proceedings of the Twenty-Eighth International Joint
  Conference on Artificial Intelligence (IJCAI 2019)}, pages 1544--1550, 2019.

\bibitem{Anglberger}
A.~J. Anglberger, N.~Gratzl, and O.~Roy.
\newblock Obligation, free choice, and the logic of weakest permissions.
\newblock {\em The Review of Symbolic Logic}, 8:807--827, 2015.

\bibitem{DBLP:conf/prima/Aucher04}
G.~Aucher.
\newblock A combined system for update logic and belief revision.
\newblock In {\em Intelligent Agents and Multi-Agent Systems, 7th Pacific Rim
  International Workshop on Multi-Agents (PRIMA 2004)}, volume 3371 of {\em
  LNCS}, pages 1--17. Springer, 2005.

\bibitem{Aumann99JGT}
R.~Aumann.
\newblock Interactive epistemology {I}: Knowledge.
\newblock {\em International Journal of Game Theory}, 28(3):263--300, 1999.

\bibitem{AumannBrand95}
R.~Aumann and A.~Brandenburger.
\newblock Epistemic conditions for {N}ash equilibrium.
\newblock {\em Econometrica}, 63:1161--1180, 1995.

\bibitem{BalbianiHerzigTroquard2008}
P.~Balbiani, A.~Herzig, and N.~Troquard.
\newblock Alternative axiomatics and complexity of deliberative stit theories.
\newblock {\em Journal of Philosophical Logic}, 37(4):387--406, 2008.

\bibitem{BMS1998}
A.~Baltag, L.~Moss, and S.~Solecki.
\newblock The logic of public announcements, common knowledge and private
  suspicions.
\newblock In {\em Proceedings of TARK'98}, pages 43--56. Morgan Kaufmann, 1998.

\bibitem{BaltagMossDEL}
A.~Baltag and L.~S. Moss.
\newblock Logics for epistemic programs.
\newblock {\em Synthese}, 139(2):165--224, 2004.

\bibitem{BaltagSmets2008}
A.~Baltag and S.~Smets.
\newblock A qualitative theory of dynamic interactive belief revision.
\newblock In {\em Proceedings of LOFT 7}, volume~3 of {\em Texts in Logic and
  Games}, pages 13--60. Amsterdam University Press, 2008.

\bibitem{BaltagSmets}
A.~Baltag and S.~Smets.
\newblock Talking your way into agreement: Belief merge by persuasive
  communication.
\newblock In {\em Proceedings of the Second Multi-Agent Logics, Languages, and
  Organisations Federated Workshops (MALLOW)}, volume 494. CEUR, 2009.

\bibitem{belnap01facing}
N.~Belnap, M.~Perloff, and M.~Xu.
\newblock {\em Facing the future: agents and choices in our indeterminist
  world}.
\newblock Oxford University Press, 2001.

\bibitem{DuboisBR}
S.~Benferhat, D.~Dubois, H.~Prade, and M.-A. Williams.
\newblock A practical approach to revising prioritized knowledge bases.
\newblock {\em Studia Logica}, 70:105--130, 2002.

\bibitem{BoolosProvability}
G.~Boolos.
\newblock {\em The Logic of Provability}.
\newblock Cambridge University Press, 1993.

\bibitem{DBLP:conf/ijcai/Boutilier93}
C.~Boutilier.
\newblock Revision sequences and nested conditionals.
\newblock In {\em Proceedings of the 13th International Joint Conference on
  Artificial Intelligence (IJCAI 1993)}, pages 519--525. Morgan Kaufmann, 1993.

\bibitem{Boutilier94}
C.~Boutilier.
\newblock Towards a logic for qualitative decision theory.
\newblock In {\em Proceedings of International Conference on Principles of
  Knowledge Representation and Reasoning (KR' 94)}, pages 75--86. AAAI Press,
  1994.

\bibitem{BrafmanTennen2}
R.~I. Brafman and M.~Tennenholtz.
\newblock On the foundations of qualitative decision theory.
\newblock In {\em Proceedings of the Thirteenth National Conference on
  Artificial Intelligence (AAAI'96)}, pages 1291--1296. AAAI Press, 1996.

\bibitem{BrafmanTennen}
R.~I. Brafman and M.~Tennenholtz.
\newblock An axiomatic treatment of three qualitative decision criteria.
\newblock {\em Journal of the ACM}, 47(3):452--482, 2000.

\bibitem{Coh90}
P.~R. Cohen and H.~J. Levesque.
\newblock Intention is choice with commitment.
\newblock {\em Artificial Intelligence}, 42:213--261, 1990.

\bibitem{DBLP:journals/ai/DarwicheP97}
A.~Darwiche and J.~Pearl.
\newblock On the logic of iterated belief revision.
\newblock {\em Artificial Intelligence}, 89(1-2):1--29, 1997.

\bibitem{PhdDeGiacomo}
G.~{De Giacomo}.
\newblock {\em Decidability of Class-Based Knowledge Representation
  Formalisms}.
\newblock PhD thesis, Universit\`a di Roma ``La Sapienza'', 1995.

\bibitem{DoyleThomason}
J.~Doyle and R.~Thomason.
\newblock Background to qualitative decision theory.
\newblock {\em The AI Magazine}, 20(2):55--68, 1999.

\bibitem{DuboisLorini}
D.~Dubois, E.~Lorini, and H.~Prade.
\newblock The strength of desires: a logical approach.
\newblock {\em Minds and Machines}, 27(1):199--231, 2017.

\bibitem{Fagin1995}
R.~Fagin, J.~Halpern, Y.~Moses, and M.~Vardi.
\newblock {\em Reasoning about Knowledge}.
\newblock MIT Press, Cambridge, Massachusetts, 1995.

\bibitem{Goranko93}
G.~Gargov and V.~Goranko.
\newblock Modal logic with names.
\newblock {\em Journal of Philosophical Logic}, 22:607--636, 1993.

\bibitem{Har00}
D.~Harel, D.~Kozen, and J.~Tiuryn.
\newblock {\em Dynamic Logic}.
\newblock MIT Press, Cambridge, Massachusetts, 2000.

\bibitem{DBLP:journals/ndjfl/Hemaspaandra96}
E.~Hemaspaandra.
\newblock The price of universality.
\newblock {\em Notre Dame Journal of Formal Logic}, 37(2):174--203, 1996.

\bibitem{Hintikka}
J.~Hintikka.
\newblock {\em Knowledge and belief: an introduction to the logic of the two
  notions}.
\newblock Cornell University Press, 1962.

\bibitem{HumberstoneFit}
I.~L. Humberstone.
\newblock Direction of fit.
\newblock {\em Mind}, 101(401):59--83, 1992.

\bibitem{IcardPacuit}
T.~F. Icard, E.~Pacuit, and Y.~Shoham.
\newblock Joint revision of beliefs and intention.
\newblock In {\em Proceedings of the Twelfth International Conference on
  Principles of Knowledge Representation and Reasoning (KR 2010)}, pages
  572--574. AAAI Press, 2010.

\bibitem{LewisPermission}
D.~Lewis.
\newblock A problem about permission.
\newblock In {\em Essays in honour of Jaakko Hintikka}, pages 163--175. 1979.

\bibitem{Fenrong}
F.~Liu.
\newblock {\em Reasoning about Preference Dynamics}.
\newblock Springer, 2011.

\bibitem{LoriniJANCL}
E.~Lorini.
\newblock Temporal {STIT} logic and its application to normative reasoning.
\newblock {\em Journal of Applied Non-Classical Logics}, 23(4):372--399, 2013.

\bibitem{LoriniIfColog}
E.~Lorini.
\newblock Logics for games, emotions and institutions.
\newblock {\em If-CoLog Journal of Logics and their Applications},
  4(9):3075--3113, 2017.

\bibitem{DBLP:conf/jelia/Lorini19}
E.~Lorini.
\newblock Reasoning about cognitive attitudes in a qualitative setting.
\newblock In {\em Proceedings of the 16th European Conference on Logics in
  Artificial Intelligence (ECAI 2019)}, volume 11468 of {\em LNCS}, pages
  726--743. Springer, 2019.

\bibitem{DBLP:conf/aiml/LutzS00}
C.~Lutz and U.~Sattler.
\newblock The complexity of reasoning with boolean modal logics.
\newblock In {\em Proceedings of the Third Conference on Advances in Modal
  logic (AiML 3)}, pages 329--348. World Scientific, 2000.

\bibitem{MarshWallaceWT}
K.~Marsh and H.~Wallace.
\newblock The influence of attitudes on beliefs: Formation and change.
\newblock In D.~Albarracin, B.~T. Johnson, and M.~P. Zanna, editors, {\em The
  Handbook of Attitudes}, pages 369--395. Lawrence Erlbaum Ass., 2005.

\bibitem{Mey99}
J.~J.~Ch. Meyer, W.~van~der Hoek, and B.~van Linder.
\newblock A logical approach to the dynamics of commitments.
\newblock {\em Artificial Intelligence}, 113(1-2):1--40, 1999.

\bibitem{PassyTinchev}
S.~Passy and T.~Tinchev.
\newblock An essay in combinatorial dynamic logic.
\newblock {\em Information and Computation}, 93:263--332, 1991.

\bibitem{Paternotte}
C.~Paternotte.
\newblock Rational choice theory.
\newblock In I.~Jarvie and J.~Zamora-Bonilla, editors, {\em SAGE Handbook for
  the Philosophy of Social Sciences}, pages 307--321. SAGE Publications Inc.,
  2011.

\bibitem{Platts}
M.~Platts.
\newblock {\em Ways of meaning}.
\newblock Routledge, and Kegan Paul, 1979.

\bibitem{Rott27}
H.~Rott.
\newblock Shifting priorities: Simple representations for 27 iterated theory
  change operators.
\newblock In D.~Makinson, J.~Malinowski, and H.~Wansing, editors, {\em Towards
  Mathematical Philosophy: Papers from the Studia Logica conference Trends in
  Logic IV}, pages 269--296. Springer, 2009.

\bibitem{Searle1979}
J.~Searle.
\newblock {\em Expression and meaning}.
\newblock Cambridge University Press, 1979.

\bibitem{Shoham2009}
Y.~Shoham.
\newblock Logical theories of intention and the database perspective.
\newblock {\em Journal of Philosophical Logic}, 38(6):633--647, 2009.

\bibitem{SpohnGames}
W.~Spohn.
\newblock How to make sense of game theory.
\newblock In {\em Philosophy of Economics}, volume~2, pages 239--270. 1982.

\bibitem{Spohn}
W.~Spohn.
\newblock Ordinal conditional functions: a dynamic theory of epistemic states.
\newblock In {\em Causation in decision, belief change and statistics}, pages
  105--134. Kluwer, 1988.

\bibitem{BenthemRevision}
J.~van Benthem.
\newblock Dynamic logic for belief revision.
\newblock {\em Journal of Applied Non-Classical Logics}, 17(2):129--155, 2007.

\bibitem{BenthemGirardRoy}
J.~van Benthem, P.~Girard, and O.~Roy.
\newblock Everything else being equal: A modal logic for ceteris paribus
  preferences.
\newblock {\em Journal of Philosophical Logic}, 38:83--125, 2009.

\bibitem{DBLP:journals/jancl/BenthemL07}
J.~van Benthem and F.~Liu.
\newblock Dynamic logic of preference upgrade.
\newblock {\em Journal of Applied Non-Classical Logics}, 17(2):157--182, 2007.

\bibitem{kooi2007}
H.~van Ditmarsch, W.~van~der Hoek, and B.~Kooi.
\newblock {\em Dynamic epistemic logic}, volume 337.
\newblock Synthese Library, Springer, 2007.

\bibitem{DBLP:journals/synthese/Ditmarsch05}
H.~P. van Ditmarsch.
\newblock Prolegomena to dynamic logic for belief revision.
\newblock {\em Synthese}, 147(2):229--275, 2005.

\bibitem{Eijck2008}
J.~van Eijck.
\newblock Yet more modal logics of preference change and belief revision.
\newblock In {\em New Perspectives on Games and Interaction}, volume~4 of {\em
  Texts in Logic and Games}, pages 81--104. Amsterdam University Press, 2008.

\bibitem{WellmanDoyle}
M.~P. Wellman and J.~Doyle.
\newblock Preferential semantics for goals.
\newblock In {\em Proceedings of the Ninth National conference on Artificial
  intelligence (AAAI'91)}, pages 698--703, 1991.

\bibitem{Woo00}
M.~Wooldridge.
\newblock {\em Reasoning about rational agents}.
\newblock MIT Press, Cambridge, 2000.

\bibitem{WrightPreference}
G.~H.~Von Wright.
\newblock {\em The logic of preference}.
\newblock Edinburgh University Press, 1963.

\bibitem{WrightPreference2}
G.~H.~Von Wright.
\newblock The logic of preference reconsidered.
\newblock {\em Theory and Decision}, 3:140--169, 1972.

\end{thebibliography}
\end{document}